\documentclass{article} 
\usepackage[final]{COLM2025/colm2025_conference}

\usepackage{microtype}
\usepackage{hyperref}
\usepackage{url}
\usepackage{booktabs}

\usepackage{lineno}

\usepackage{amsthm, amsmath, amssymb, bbm, mathabx}
\allowdisplaybreaks 
\usepackage{graphicx}
\usepackage{enumerate}
\usepackage{pifont}
\usepackage{booktabs}
\usepackage{multirow}
\usepackage{multicol}
\usepackage{stfloats}
\usepackage{xspace}
\usepackage{thmtools}
\usepackage{thm-restate}
\usepackage{hyperref}
\usepackage{cleveref}
\usepackage{anyfontsize} 
\usepackage{makecell}
\usepackage{hhline}
\usepackage{colortbl}
\usepackage{xpatch}
\usepackage{scalerel,stackengine}
\usepackage{sidecap}
\usepackage{float}
\usepackage{mdframed}
\usepackage{listings}

\makeatletter
\def\@fnsymbol#1{\ensuremath{\ifcase#1\or *\or \dagger\or \ddagger\or
  \mathsection\or \mathparagraph\or \|\or \diamond \or **\or \dagger\dagger
  \or \ddagger\ddagger \else\@ctrerr\fi}}
\makeatother

\makeatletter
\newcommand{\printfnsymbol}[1]{%
  \textsuperscript{\@fnsymbol{#1}}%
}
\makeatother

\newcommand{\para}[1]{\paragraph{#1}}

\newtheorem{theorem}{Theorem}
\newtheorem{lemma}{Lemma}

\newtheorem{definition}{Definition}

\crefname{condition}{Condition}{Conditions}

\crefname{assumption}{Assumption}{Assumptions}

\theoremstyle{definition}

\newfloat{textbox}{tbp}{lob}
\floatname{textbox}{Text Box}
\crefname{textbox}{Text Box}{Text Boxes}

\newmdenv[
  linewidth=1pt,
  skipabove=10pt,
  skipbelow=10pt,
  backgroundcolor=gray!10,
  roundcorner=5pt,
  leftmargin=10pt,
  rightmargin=10pt
]{textframe}

\lstdefinelanguage{text}{
  basicstyle=\ttfamily\scriptsize\selectfont\color[rgb]{0,0,0},
  keywordstyle=\color[rgb]{1, 0., 0.},
  sensitive=true,
  morecomment=[l]{//},
  morecomment=[s]{/*}{*/},
  commentstyle=\color[gray]{0.4},
  morestring=[b]",
  morestring=[b]',
  columns=fullflexible,
  numbers=none,
  framerule=0pt,           
  framexleftmargin=0em,    
  framexrightmargin=0em,   
  xleftmargin=0pt,
  xrightmargin=0pt,        
  frame=none,
  breaklines=true,
  framesep=0pt,            
  resetmargins=true        
}

\newcommand{\abs}[1]{\left| #1 \right|}
\newcommand{\norm}[1]{\left\| #1 \right\|}

\newcommand\myeqi{\mathrel{\stackrel{\makebox[0pt]{\mbox{\normalfont\tiny (i)}}}{=}}}
\newcommand\myeqii{\mathrel{\stackrel{\makebox[0pt]{\mbox{\normalfont\tiny (ii)}}}{=}}}
\newcommand\myeqiii{\mathrel{\stackrel{\makebox[0pt]{\mbox{\normalfont\tiny (iii)}}}{=}}}

\newcommand\mylei{\mathrel{\stackrel{\makebox[0pt]{\mbox{\normalfont\tiny (i)}}}{\le}}}
\newcommand\myleii{\mathrel{\stackrel{\makebox[0pt]{\mbox{\normalfont\tiny (ii)}}}{\le}}}
\newcommand\myleiii{\mathrel{\stackrel{\makebox[0pt]{\mbox{\normalfont\tiny (iii)}}}{\le}}}

\newcommand{\wt}[1]{\widetilde{#1}}
\newcommand{\wh}[1]{\widehat{#1}}

\newcommand{\cD}{\mathcal{D}}

\newcommand{\cL}{\mathcal{L}}

\newcommand{\cN}{\mathcal{N}}

\newcommand{\cP}{\mathcal{P}}

\newcommand{\cX}{\mathcal{X}}
\newcommand{\cY}{\mathcal{Y}}

\newcommand{\bbE}{\mathbb{E}}

\newcommand{\bbP}{\mathbb{P}}

\newcommand{\bbR}{\mathbb{R}}

\newcommand{\II}{\mathbbm{1}} 

\newcommand{\ee}{\textup{e}}
\newcommand{\KL}{\mathsf{KL}}

\newcommand{\poly}{\mathsf{poly}}

\newcommand{\Unif}{\mathsf{Uniform}}

\newcommand{\subG}{\mathsf{sub\text{-}Gaussian}}

\newcommand{\Roma}[1]{\uppercase\expandafter{\romannumeral#1}}

\newcommand{\tref}{{\mathsf{ref}}}

\newcommand{\piref}{\pi_{\tref}}
\newcommand{\sg}[1]{\mathsf{sg}[#1]}
\newcommand{\samp}{\ensuremath{\pi^{\mathsf{s}}}}

\newcommand{\ipo}{\ensuremath{\mathsf{IPO}}}

\newcommand{\dualgap}{\ensuremath{\mathsf{DualGap}}}

\newcommand{\oipoone}{{\texttt{OIPO1}}\xspace}
\newcommand{\oipotwo}{{\texttt{OIPO2}}\xspace}
\newcommand{\nmd}{{\texttt{NMD}}\xspace}
\newcommand{\nmdpg}{{\texttt{NMDPG}}\xspace}
\newcommand{\mpo}{{\texttt{MPO}}\xspace}
\newcommand{\eg}{{\texttt{EGPO}}\xspace}

\newcommand{\etatheory}{\ensuremath{\eta_{\textup{theory}}}}
\newcommand{\etaoptimizer}{\ensuremath{\eta_{\textup{optimizer}}}}

\newcommand{\red}[1]{\textcolor{red!80}{#1}}

\newcommand{\notprovided}{\textcolor{gray!80}{Not provided}}
\newcommand{\NA}{\textcolor{gray!80}{N/A}}

\usepackage[colorinlistoftodos, textwidth=18mm]{todonotes}

\def\shownotes{1}
\ifnum\shownotes=0
\newcommand{\todorz}[1]{}
\newcommand{\todorzout}[1]{}
\newcommand{\simon}[1]{}
\newcommand{\mf}[1]{}
\else
\newcommand{\todorz}[1]{\todo[color=blue!10, inline]{\small RZ: #1}}
\newcommand{\todorzout}[1]{\todo[color=blue!10]{\scriptsize RZ: #1}}

\newcommand{\simon}[1]{\textcolor{cyan}{[Simon: #1]}}
\newcommand{\mf}[1]{\textcolor{magenta}{[Maryam: #1]}}
\fi

\definecolor{darkblue}{rgb}{0, 0, 0.5}
\hypersetup{colorlinks=true, citecolor=darkblue, linkcolor=darkblue, urlcolor=darkblue}

\title{Extragradient Preference Optimization (\eg):\\
Beyond Last-Iterate Convergence for Nash Learning from Human Feedback}

\author{
Runlong Zhou\\
University of Washington\\
\texttt{vectorzh@cs.washington.edu}\\
\And
Maryam Fazel\\
University of Washington\\
\texttt{mfazel@uw.edu}\\
\And
Simon S. Du\\
University of Washington\\
\texttt{ssdu@cs.washington.edu}
}

\begin{document}

\ifcolmsubmission
\linenumbers
\fi

\maketitle

\begin{abstract}
Reinforcement learning from human feedback (RLHF) has become essential for improving language model capabilities, but traditional approaches rely on the assumption that human preferences follow a transitive Bradley-Terry model.
This assumption fails to capture the non-transitive nature of populational human preferences.
Nash learning from human feedback (NLHF), targeting non-transitive preferences, is a problem of computing the Nash equilibrium (NE) of the two-player constant-sum game defined by the human preference.
We introduce Extragradient preference optimization (\eg), a novel algorithm for NLHF achieving last-iterate linear convergence to the NE of KL-regularized games and polynomial convergence to the NE of original games, while being robust to noise.
Unlike previous approaches that rely on nested optimization, we derive an equivalent implementation using gradients of an online variant of the identity preference optimization (IPO) loss, enabling more faithful implementation for neural networks.
Our empirical evaluations demonstrate \eg's superior performance over baseline methods when training for the same number of epochs, as measured by pairwise win-rates using the ground truth preference.
These results validate both the theoretical strengths and practical advantages of \eg for language model alignment with non-transitive human preferences.
To facilitate research in the field of NLHF, the code is publicly released.\footnote{\url{https://github.com/zhourunlong/EGPO}}
\end{abstract}

\section{Introduction}

Reinforcement learning from human feedback (RLHF, \citet{Christiano2017DeepRL,Ziegler2019FineTuningLM}) is a prevalent and crucial technique for improving the natural language understanding and generation capabilities of large language models (LLMs).
While directly collecting absolute \emph{reward} data from human annotators is difficult, comparing responses to obtain \emph{preference} data is more reasonable.
RLHF aligns LLMs with human preferences through fine-tuning with proximal policy optimization (PPO, \citet{schulman2017proximal}) using a reward model trained from preference signals.
The reward modeling stage assumes human preferences follow the Bradley-Terry (BT) model \citep{BTmodel}, allowing response $y$ to be assigned a scalar reward $r(x, y)$ given prompt $x$.
The preference $\cP (y \succ y' | x)$ (the fraction of human annotators believing $y$ is better than $y'$ given prompt $x$) equals $\sigma (r(x, y) - r(x, y'))$, where $\sigma (t) = 1 / (1 + \exp(-t))$.
Following this formulation, direct preference optimization (DPO, \citet{rafailov2023direct}) utilizes the closed-form solution for the policy in the PPO training stage to bypass the reward modeling stage and directly fine-tune the policy model.

However, the scalar reward model assumption has limitations, most notably the 
transitivity between responses: if $A$ is preferred over $B$, and $B$ is preferred over $C$, then $A$ must be preferred over $C$.
While this may be true for \emph{individuals}, it often contradicts evidence at an \emph{aggregated, population} level \citep{may1954intransitivity}.
Readers can refer to \citet{munos2023nash} for additional limitations of transitive preferences.
\citet{munos2023nash} first formally considered \emph{non-transitive}, general preferences in RLHF, naming it Nash learning from human feedback (NLHF) where the goal is to find the Nash equilibrium (NE, \citet{nash1950equilibrium}) of the preference.
Naturally, the preference should satisfy $\cP (y \succ y' | x) + \cP (y' \succ y | x) = 1$, which induces a two-player constant-sum game; see section 3.1 for an overview.
Consequently, the win-rate of the NE policy is at least $50\%$ against \emph{any other} policy.

Solving for an (approximate) NE requires several desirable properties in LLM applications; see \Cref{sec:egpo_remarks} for a more detailed discussion.
The most important is achieving \emph{\textbf{last-iterate convergence}} to the NE, which guarantees that the final policy in the training process satisfies a certain approximation requirement.
In contrast, average-iterate convergence requires the output policy to average over all historical policies, which is prohibitive as storing and performing forward passes of all historical LLMs is space and time inefficient.

Additionally, \emph{\textbf{convergence rate}} and \emph{\textbf{robustness to sampling noise}} 
are crucial, as human-collected data is costly and noisy.
In the online RLHF setting, a faster convergence rate directly translates to fewer rounds of data collection while achieving the same performance.
Desired convergence rates are linear (e.g., $0.9^T$) for NE with regularization and polynomial (e.g., $1/T$) for the original NE.
Convergence is usually analyzed with exact updates, so ideally when accounting for noise, its rate should remain unchanged, with the only difference being convergence to a constant term \emph{scaling with the noise}.

Finally, implementation for \emph{\textbf{general parametric policies}} (neural networks) should be as faithful as possible to the theoretical version for tabular policies.
This is important because in RLHF, most algorithms \citep{rafailov2023direct,Azar2023AGT,munos2023nash,swamy2024minimaximalist,rafailov2024rqlanguagemodel,calandriello2024humanalignmentlargelanguage,zhang2024iterative,shi2025the} are originally designed for tabular policies, 
so extension to neural networks inevitably introduces mismatches between theory and implementation.
This requirement first prohibits the direct parametrization approach (e.g., OGDA, see \citet{wei2020linear}), namely using $\theta_{x, y}$ to directly represent the probability of outputting $y$ given input $x$, as it requires projection to probability simplex which is intractable for neural networks.
Secondly, the theoretical algorithm should avoid nested optimization, e.g., $\theta^{(t+1)} = \arg\min_\theta \cL_{\textup{inner}} (\theta; \theta^{(t)})$, which is adopted by \citet{munos2023nash,ye2024online,Rosset2024DirectNO,wu2024self,zhang2024iterative,wang2024magneticpreferenceoptimizationachieving,zhang2025improving}.
In practice we can only perform a small number of gradient descent steps on $\cL_{\textup{inner}}$ to approximately compute $\wh{\theta}^{(t+1)}$, so the error between $\wh{\theta}^{(t+1)}$ and $\theta^{(t+1)}$ will accumulate and affect the final convergence. 

\subsection{Our contributions}

Our contributions satisfy the aforementioned desired properties and demonstrate improved performance in LLM alignment experiments, summarized as follows:

$\bullet$ \textbf{Theoretical soundness.}
We propose Extragradient preference optimization (\eg) for NLHF, which achieves last-iterate linear convergence to the NE of the KL-regularized game and last-iterate polynomial convergence to the NE of the original game.
When gradient updates contain sub-Gaussian noises, only the final convergent value changes by an additive amount scaling with the noise variance, while the convergence rate remains unchanged.
We deliver detailed comparisons with previous works in \Cref{tab:comparison} and \Cref{sec:egpo_remarks}.

$\bullet$ \textbf{Faithful implementation.}
We derive an equivalent implementation of \eg using gradients of an online variant of identity preference optimization (IPO) loss, eliminating the \textbf{\emph{nested optimization}} widely adopted in previous NLHF works.
This equivalence extends to a broader range of algorithms, and we demonstrate its efficacy compared to approximate nested optimization.

$\bullet$ \textbf{Improved performance on benchmarks.}
We evaluate \eg against several baselines by training for an identical number of epochs and computing pairwise win-rates using the ground truth preference.
Results confirm the theoretical advantages of \eg.

\subsection{Paper overview}

We first introduce basic concepts for NLHF in \Cref{sec:preliminaries}.
Next, we present our main algorithm, Extragradient preference optimization (\eg), along with an equivalent online IPO formulation, in \Cref{sec:algs}.
Finally, we demonstrate the efficacy of \eg through numerical simulations and language model alignments in \Cref{sec:exp}.
Proofs and additional related work are in the appendices.

\colorlet{shadecolor}{gray!40}
\begin{table*}[t] 
    \centering
    \small
    \resizebox{0.99\linewidth}{!}{%
        \renewcommand{\arraystretch}{1.5}
        \begin{tabular}{|c|c|c|c|c|c|}
            \hline 
            \textbf{Algorithm} &  \scriptsize \makecell{\textbf{Convergence to} \\ \textbf{Regularized QRE}} & \textbf{Range of} $\eta$ & \scriptsize \makecell{\textbf{Last-iterate} \\ \textbf{Convergence}} & \scriptsize \makecell{\textbf{$\sigma^2$-noise} \\ \textbf{Robustness}} & \scriptsize \makecell{\textbf{Convergence to} \\ \textbf{Original $\varepsilon$-NE}}\\
            
            \hhline{|=|=|=|=|=|=|}
            Online Mirror Descent & $\wt{O} (1 / T)$ & $\eta \le O (1 / \beta)$ & No & \notprovided & $\wt{O} (1 / \varepsilon^2)$ iterations  \\
            
            \hline
            \multirow{2}{*}{\makecell{Nash-MD \citep{munos2023nash}\\ MTPO\footnotemark \citep{shani2024multi}}} &  $ \wt{O} ((1 - \eta \beta)^T + \eta / \beta) $ & $\eta \le O (1 / \beta)$ & \multirow{2}{*}{\textbf{Yes}} & \multirow{2}{*}{\notprovided } & \multirow{2}{*}{\notprovided } \\
            \cline{2-3}
            & $\wt{O} (1 / T)$ & $\eta = \wt{\Theta} (1 / (\beta T))$ & & &  \\

            \hline
            \makecell{SPO\footnotemark \citep{swamy2024minimaximalist}\\ SPPO\footnotemark \citep{wu2024self}} & \notprovided & \NA & No & \notprovided & $\wt{O} (1 / \varepsilon^2)$ iterations \\

            \hline
            \makecell{INPO\footnotemark \\ \citet{zhang2024iterative}} & $\wt{O} (1/T)$ & $\eta_t = \Theta (1 / (\beta t))$ & \textbf{Yes} & \notprovided  & \notprovided  \\

            \hline
            \makecell{MPO \\ \citet{wang2024magneticpreferenceoptimizationachieving}} & $\wt{O} ((\frac{1}{1 + \eta \beta})^T)$ \textbf{(linear)} & $\eta \le O (\beta)$ & \textbf{Yes} & \notprovided  & $\wt{O} (1 / \varepsilon^2)$ iterations \\

            \hline
            \makecell{ONPO \\ \citet{zhang2025improving}} & \notprovided  & \NA & No & \notprovided  & $\wt{O} (1 / \varepsilon)$ \textbf{iterations} \\

            \hline
            \rowcolor{shadecolor} \Gape[0pt][2pt]{\makecell{\normalsize\eg \\ This work}} & $\wt{O} ((1 - \eta \beta)^T)$ \textbf{(linear)} & $\eta \le O (1 / (\beta \lor 1))$ & \textbf{Yes} & $+ \wt{O}(\sigma^2 / (\eta \beta^2))$ & $\wt{O} (1 / \varepsilon)$ \textbf{iterations} \\
            \hline
        \end{tabular}
    }
    \caption{
    Comparison of convergence rates across different algorithms for NLHF.
    \textbf{Convergence to regularized QRE:} Measured by the KL divergence between the current policy and the regularized QRE, or the duality gap.
    Here, $\beta$ represents the regularization coefficient, $\eta$ is the learning rate, and $T$ denotes the number of updates.
    \textbf{Range of $\eta$:} The condition on $\eta$ for the convergence to hold.
    \textbf{Last-iterate convergence:} "Yes" indicates the convergence rate applies to the final policy.
    "No" indicates it applies only to the average of all generated policies.
    \textbf{$\sigma^2$-noise robustness:} When updates are estimated and contain sub-Gaussian noise with variance proxy $\sigma^2$, this column shows the resulting impact on convergence.
    The optimal property is the addition of only a constant term scaling with $\sigma^2$ without affecting the main convergence term.
    See \Cref{sec:baseline_empirical} for more discussions.
    \textbf{Convergence to original $\varepsilon$-NE:} Number of iterations required to reach an $\varepsilon$-NE of the original matrix game, measured by duality gap.
    See \Cref{sec:baseline_original_ne} for more discussions.\\
    \noindent\rule{\textwidth}{0.4pt}
    $^2$\small{MTPO could be viewed as Nash-MD for multi-turn contextual bandits.}\\
    $^3$\small{SPO is the only algorithm in this table capable of handling Markov decision processes (as opposed to bandits).}\\
    $^4$\small{SPPO could be viewed as a special case of SPO applied to contextual bandits.}\\
    $^5$\small{INPO assumes that $\pi^{(t)}$ does not deviate significantly from $\pi_\text{ref}$ (see their Assumption~A) in any trajectory of the update. However, this assumption is not verified. In fact, verifying or achieving this assumption is not straightforward (see the proofs of Theorems~1, 4 and 6 in \citet{shi2025the}).}
    }
    \label{tab:comparison}
\end{table*}
\section{Related works} \label{sec:rel}

Due to page limit, we defer related works on general RLHF to \Cref{sec:additional_rel}.

\para{Game-theoretic RLHF.}
A growing body of research \citep{wang2023rlhf,munos2023nash,swamy2024minimaximalist,ye2024online,Rosset2024DirectNO,calandriello2024humanalignmentlargelanguage,zhang2024iterative,wu2024self,wang2024magneticpreferenceoptimizationachieving,zhang2025improving,tang2025game} examines RLHF from a game-theoretic perspective.
These works focus on finding the Nash equilibrium (NE) of human preferences, with several capable of handling non-transitive preferences.
Self-play preference optimization methods \citep{swamy2024minimaximalist,wu2024self} offer average-iterate convergence guarantees on the duality gap.
Nash-MD \citep{munos2023nash}, MTPO \citep{shani2024multi}, and MPO \citep{wang2024magneticpreferenceoptimizationachieving} are algorithms with stronger \emph{\textbf{last-iterate}} convergence guarantees on the KL divergence between the learned policies and the Nash equilibria.
Last-iterate convergence guarantees are crucial for applications using large neural networks, as storing mixtures of all historical models is impractical.

\para{Computing equilibria in two-player zero-sum matrix games.}
Two-player zero-sum games closely relate to the game-theoretic formulation of RLHF.
Online mirror descent (OMD) \citep{Cesa-Bianchi_Lugosi_2006,lattimore2020bandit}, designed to solve online convex learning problems, naturally applies to finding the NE of the preference.
However, OMD only achieves average-iterate convergence.
Optimistic gradient descent ascent (OGDA, see \citet{wei2020linear}) achieves linear last-iterate convergence when the policy class is directly parameterized in the probability simplex (constrained class).
Though favorable for tabular settings, this result is difficult to generalize to neural networks due to the direct parameterization.
\citet{cen2021fast} study the KL-regularized game setting, which precisely models game-theoretical RLHF problems.
The authors show that two instantiations of Extragradient methods both achieve linear last-iterate convergence when the policy class is tabular softmax (unconstrained class).
We extend one of their algorithms, predictive update (PU), to the gradient estimation setting and the practical neural network setting.
Other related algorithms include (optimistic) multiplicative weight update \citep{freund1999adaptive,bailey2018multiplicative,daskalakis2018last,cen2022faster}, Nesterov's excessive gap technique \citep{daskalakis2011near}, optimistic mirror descent \citep{rakhlin2013optimization}, and magnetic mirror descent \citep{sokota2022unified}.

\section{Preliminaries} \label{sec:preliminaries}

\para{Notations.}
For any set $\cX$, $\Delta(\cX)$ represents the set of probability distributions over $\cX$.
$\sg{}$ denotes the stopping-gradient operator, which treats the quantity inside it as a constant (see \Cref{eq:ipo_est}).
We use $\II_{n, m}$ to denote an $n \times m$ matrix with all entries equal to $1$, and omit the subscripts when the dimension is clear from context.
We use $\wt{O}, \wt{\Theta}, \wt{\Omega}$ to hide $\poly\log (\abs{\cY} T / (\varepsilon \eta \beta))$ factors.

\para{Prompts and responses.}
In RLHF, we denote $\cX$ as the prompt space and $\cY$ as the response space.
To simplify notation, we assume that $\abs{\cX} = 1$ as in \citet{munos2023nash,zhang2024iterative}.
The statements and proofs can be easily extended to larger $\cX$.
Thus, we omit the prompts and focus on the responses.

\para{Policies.}
A policy $\pi: \cY \to [0, 1]$ maps each response to a probability.
Under the \emph{tabular softmax parametrization} common in previous works \citep{rafailov2023direct,Azar2023AGT,munos2023nash,swamy2024minimaximalist}, $\pi$ is parameterized by $\theta\in\mathbb R^{\vert\cY\vert }$: for any $y \in \cY$,
\begin{align*}
    \pi_\theta(y) = \frac{\exp(\theta_y)}{\sum_{y' \in \cY} \exp(\theta_{y'})}\;.
\end{align*}
Let $\theta_\clubsuit^\diamondsuit \in \bbR^{\abs{\cY}}$, we denote $\pi_\clubsuit^\diamondsuit := \pi_{\theta_\clubsuit^\diamondsuit}$, where $\clubsuit$ and $\diamondsuit$ could be any symbol.

\para{RLHF.}
We defer concepts of RLHF to \Cref{sec:additional_prelim}.

\subsection{Nash learning from human feedback (NLHF)}
In general, human preferences cannot be assumed to be transitive \citep{may1954intransitivity}.
Thus, a global ordering based on an implicit reward function (e.g., in Bradley-Terry model) has significant limitations \citep{munos2023nash,wang2024magneticpreferenceoptimizationachieving}.

\para{Non-transitive preference.}
Define the preference as
\begin{align*}
    \cP (y \succ y') := \bbP [y \textup{ is preferred over } y' \textup{ by human annotators}].
\end{align*}
It satisfies $\cP (y \succ y') + \cP (y' \succ y) = 1$.
Specifically, $\cP (y \succ y) = \frac{1}{2}$.
For notational ease, we denote $\cP_{y,y'} := \cP (y \succ y')$, $\pi_y := \pi (y)$ as a matrix and a vector, respectively, and
\begin{align*}
    \cP (y \succ \pi') := \bbE_{y' \sim \pi'} \cP (y \succ y') = \cP \pi', \quad
    \cP (\pi \succ \pi') := \bbE_{y \sim \pi, y' \sim \pi'} \cP (y \succ y') = \pi^\top \cP \pi'.
\end{align*}

\para{RLHF as a two-player constant-sum matrix game.}
We aim to find a policy $\pi^\star$ that is preferred over any other (adversarial) policy, so we define
\begin{align*}
    & V (\pi, \pi') := \pi^\top \cP \pi', \\
    & \pi^\star = \arg\max_\pi \min_{\pi'} \cP (\pi \succ \pi') = \arg\max_\pi \min_{\pi'} \pi^\top \cP \pi'.
\end{align*}
The second player receives a payoff of $\cP (\pi' \succ \pi) = 1 - \cP (\pi \succ \pi')$.
This solution is the Nash equilibrium (NE) for this game by the Minimax theorem \citep{Neumann1928ZurTD}.


\para{Regularized game.}
The regularized game and value are defined with respect to a reference policy $\pi_\tref$:
\begin{align*}
    &V_\beta (\pi_1, \pi_2) := \pi_1^\top \cP \pi_2 - \beta \KL (\pi_1 || \pi_\tref) + \beta \KL (\pi_2 || \pi_\tref), \\
    &\theta_1^\star = \arg\max_{\theta_1} \min_{\theta_2} V_\beta (\pi_1, \pi_2).
\end{align*}
    
We denote $\pi_\beta^\star$ as the quantal response equilibrium (QRE, \citet{McKelvey1995QuantalRE}) which satisfies $\theta_\beta^\star = \theta_\tref + \frac{\cP \pi_\beta^\star}{\beta} + C \II_{\abs{\cY}}$ (see \Cref{eq:closed_form}).
We can set $C = 0$ without loss of generality. 
Thus, we aim to solve a multivariate equation for $\theta$:
\begin{align}
    \theta = \theta_\tref + \frac{\cP \pi_\theta}{\beta}. \label{eq:main_equation}
\end{align}

\para{Duality gap.}
The duality gap of $\pi$ in the original matrix game is defined as
\begin{align*}
    \dualgap (\pi) = \max_{\pi'} V (\pi', \pi) - \min_{\pi''} V (\pi, \pi'').
\end{align*}
The duality gap of $\pi$ in the regularized matrix game is defined as
\begin{align*}
    \dualgap_\beta (\pi) = \max_{\pi'} V_\beta (\pi', \pi) - \min_{\pi''} V_\beta (\pi, \pi'').
\end{align*}
Duality gaps are non-negative, reaching $0$ if and only if the policy is an NE/QRE.
\section{Algorithms} \label{sec:algs}

We present the main contributions of this work: an Extragradient method (\eg) for NLHF with theoretical last-iterate convergence guarantees, as well as its online IPO formulation for practical implementation.
\emph{\textbf{The final algorithm follows the update in \Cref{eq:ipo_est,eq:ipo_upd}.}}
We now have a practical \emph{single-step} optimization method (see \Cref{sec:ipo_remarks}) faithful to a theoretical algorithm, which has significant implications for the field of RLHF.

\subsection{Extragradient preference optimization (\eg)}

We \emph{\textbf{generalize}} the predictive update (PU) algorithm in \cite{cen2021fast} to the setting of practical (\textbf{empirical}) algorithms by introducing noise terms from the estimation of $\cP \pi$: 
\begin{align}
    \theta^{(t+1/2)} &= (1 - \eta \beta) \theta^{(t)} + \eta \beta \left(\theta_\tref + \frac{\cP \pi^{(t)} + \epsilon^{(t)}}{\beta} \right), \label{eq:extragrad_est} \\
    \theta^{(t+1)} &= (1 - \eta \beta) \theta^{(t)} + \eta \beta \left(\theta_\tref + \frac{\cP \pi^{(t+1/2)} + \epsilon^{(t+1/2)}}{\beta} \right). \label{eq:extragrad_upd}
\end{align}
Here for $i = t, t + 1/2$, we assume that conditioning on $\pi^{(i)}$, $\bbE [\epsilon^{(i)}] = \boldsymbol{0}$, 
for $y \in \cY$, all $(\epsilon^{(i)})_y$s are independent, and $(\epsilon^{(i)})_y \sim \subG (\sigma^2)$ (see \Cref{def:subG}).
This \textbf{practical} update generalizes the \textbf{exact} update 
(corresponding to $\sigma^2 = 0$).

The intuition is that when we perform \emph{implicit updates} $\theta^{(t+1)} = (1 - \eta \beta)\theta^{(t)} + \eta \beta (\theta_{\tref} + \frac{\cP \red{\pi^{(t+1)}}}{\beta})$, $\KL (\pi_\beta^\star || \pi^{(t)})$ converges to $0$ linearly (see Proposition 1 in \citet{cen2021fast}).
Thus, $\pi^{(t+1/2)}$ serves as an estimation for $\pi^{(t+1)}$ on the RHS in the implicit update.

We emphasize that \textbf{there is no \emph{preference modeling} in our NLHF framework, hence $\cP$ is the ground truth preference and can be accessed by querying human annotators with $(x, y, y')$ triplets}.

In this section, we analyze the convergence properties of this Extragradient method, which we call \textbf{Extragradient preference optimization} (\eg).

\subsubsection{Theoretical guarantees for \eg}
We first present \Cref{thm:extragrad_short} (with full statement in \Cref{thm:extragrad}), which describes the convergence rate of \eg.
For \emph{any policy generated throughout the process}, including $\pi^{(t)}$ and $\pi^{(t+1/2)}$, linear convergence to a constant term scaling with $\sigma^2$ is guaranteed.
This demonstrates \emph{\textbf{last-iterate}} convergence.
For \textbf{exact} updates, \eg converges linearly to the QRE of the regularized game.

\begin{theorem} \label{thm:extragrad_short}
For any initialization $\theta^{(0)}$, following the update rules defined by \Cref{eq:extragrad_est,eq:extragrad_upd} and setting $\eta \le \frac{1}{\beta + 3}$, we have that for any $T \ge 1$,
\begin{align*}
\bbE[\KL (\pi_\beta^\star || \pi^{(T)})] &\le \KL (\pi_\beta^\star || \pi^{(0)}) (1 - \eta \beta)^T + \frac{4 \sigma^2 \log (3 \abs{\cY})}{\beta},  \\
\bbE[\KL (\pi^{(T)} || \pi_\beta^\star)] &\le \frac{2 \KL (\pi_\beta^\star || \pi^{(0)})}{\eta \beta} (1 - \eta \beta)^T + \frac{8 \sigma^2 \log (3 \abs{\cY})}{\eta \beta^2}, \\
\bbE[\dualgap_\beta (\pi^{(T)})] &\le \left( \frac{2}{\beta} + \frac{4}{\eta} \right) \KL (\pi_\beta^\star || \pi^{(0)}) (1 - \eta \beta)^T + \left( \frac{8}{\beta^2} + \frac{16}{\eta \beta} \right) \sigma^2 \log (3 \abs{\cY}).
\end{align*}
Under the \textbf{exact update} scheme where $\sigma^2 = 0$, all the expectations are removed.
\end{theorem}

Next, \Cref{thm:extragrad_unreg} shows that without algorithm modification, \textbf{exact} \eg can achieve an $\varepsilon$-NE of the \emph{unregularized} game in $\wt{O} (1 / \varepsilon)$ steps through simple instantiations.
This also demonstrates last-iterate convergence.

\begin{theorem} \label{thm:extragrad_unreg}
Consider the \textbf{exact update} scheme, where $\sigma^2 = 0$.
By setting $\pi^{(0)} = \pi_\tref = \Unif (\cY)$, $\beta = \frac{\varepsilon}{4 \log \abs{\cY}}$, and $\eta = \frac{1}{\beta + 3}$, we have that for any $T \ge \wt{\Omega} (1 / \varepsilon)$,
\begin{align*}
    \dualgap (\pi^{(T)}), \dualgap (\pi^{(T+1/2)}) &\le \varepsilon.
\end{align*}
\end{theorem}

The proofs of \Cref{thm:extragrad_short,thm:extragrad_unreg} are deferred to \Cref{sec:proof_extragrad}.

\subsubsection{Remarks} \label{sec:egpo_remarks}

Now we make several remarks about the theoretical results.

\para{From the pure optimization perspective.}
We extend the results of \citet{cen2021fast} (corresponding to the \textbf{exact update} version of \Cref{eq:extragrad_kl,eq:extragrad_est_kl,eq:extragrad_est_rev_kl,eq:extragrad_est_dualgap}) by providing guarantees for $\KL (\pi^{(T)} || \pi_\beta^\star)$ (\Cref{eq:extragrad_rev_kl}) and $\dualgap_\beta (\pi^{(T)})$ (\Cref{eq:extragrad_dualgap}), and addressing \textbf{empirical updates}.
This extension is achieved by replacing \Cref{eq:core1} (Equation (23) in \citet{cen2021fast}) with \Cref{eq:core2}, and applying properties of sub-Gaussian random variables (e.g., \Cref{lem:sub_gaussian_vec_inf_norm_square}).
Note that $\sigma^2$ appears only in the \emph{constant terms}, leaving the linear convergence in the main terms unaffected.
From Pinsker's inequality (\Cref{lem:pinsker}), an upper bound on KL divergence implies an upper bound on \emph{squared} L1 distance, making \Cref{thm:extragrad} a strong guarantee.
This result indicates that \eg is robust and stable with respect to noise in the updates.

\para{In the context of NLHF.}
We compare our results with prior NLHF algorithms \citep{munos2023nash,swamy2024minimaximalist,wu2024self,zhang2024iterative,wang2024magneticpreferenceoptimizationachieving,zhang2025improving}, shown in \Cref{tab:comparison}.
We emphasize the following points:

$\bullet$ \eg (and MPO) achieves linear convergence to regularized QRE, significantly faster than the $1/T$ convergence of other algorithms.
When $\beta \to 0$, the optimal rate of \eg is $(1 - \beta)^T$, while that of MPO is $(1 - \beta^2)^T$.
To achieve an $\varepsilon$-QRE, \eg takes only $\log (1 / \varepsilon) / \beta$ steps while MPO takes $\log (1 / \varepsilon) / \beta^2$ steps.

$\bullet$ \eg (and Nash-MD, INPO, MPO) demonstrates last-iterate convergence, which is crucial in practice as mixing a large number of models is often infeasible.

$\bullet$ \eg supports the analysis of \textbf{empirical updates}, while it remains unclear whether other algorithms can provide similar guarantees. We add more discussions on the effect of empirical updates in \Cref{sec:baseline_empirical}.

\subsection{Online IPO formulation for \eg}

We present our findings on the equivalence between \eg and an online variant of identity preference optimization (IPO, \citet{Azar2023AGT}), inspired by insights from \citet{calandriello2024humanalignmentlargelanguage}.
We further explore relationships with other NLHF algorithms in \Cref{sec:ipo_nlhf}, as these connections are essential for practical implementation.

\para{Generalized IPO.}
Define a generalized IPO loss using separate distributions for $(y, y')$ and $y''$ (here we assume they are independent of $\theta$):
\begin{align*}
    \cL_\ipo (\theta; \rho, \mu)
    = \bbE_{(y, y') \sim \rho} \left[ \left( \log \frac{\pi_\theta(y) \pi_\tref (y')}{\pi_\theta (y') \pi_\tref (y)} - \frac{1}{\beta} \bbE_{y'' \sim \mu} [\cP (y \succ y'') - \cP (y' \succ y'')] \right)^2 \right].
\end{align*}

An \emph{\textbf{online IPO}} \citep{calandriello2024humanalignmentlargelanguage} algorithm is one where at least one of $\rho$ and $\mu$ is instantiated by the current policy, $\pi_\theta$.

Define $\samp := \Unif(\cY) \times \Unif(\cY)$.
We argue that the update defined by
\begin{align}
    \theta^{(t+1/2)} &= \theta^{(t)} - \underbrace{\frac{\etatheory \beta \abs{\cY}}{4}}_{=: \etaoptimizer} \nabla_\theta \cL_\ipo (\theta^{(t)}; \samp, \red{\sg{\pi^{(t)}}}), \label{eq:ipo_est} \\
    \theta^{(t+1)} &= \theta^{(t)} - \frac{\etatheory \beta \abs{\cY}}{4} \nabla_\theta \cL_\ipo (\theta^{(t)}; \samp, \pi^{(t+1/2)}), \label{eq:ipo_upd}
\end{align}
where $\etatheory$ is the $\eta$ in \Cref{thm:extragrad} and $\etaoptimizer$ is the actual learning rate used by the optimizer in contemporary machine learning frameworks, is equivalent to \eg (\Cref{eq:extragrad_est,eq:extragrad_upd}).
The justification is deferred to \Cref{sec:proof_eg_oipo}.

In practice, we use finite samples to approximate the gradient of online IPO loss.
The following theorem gives such a \textbf{population IPO loss}.
Its proof is deferred to \Cref{sec:proof_oipo}.
\begin{theorem} \label{thm:pop_ipo_loss}
Define
\begin{align}
    \wh{\cL} (\theta; \samp, \mu) := \bbE_{(y, y') \sim \samp, \red{y'' \sim \mu}} \left[ \left( \log \frac{\pi_\theta(y) \pi_\tref (y')}{\pi_\theta (y') \pi_\tref (y)} - \frac{I (y, \red{y''}) - I (y', \red{y''})}{\beta} \right)^2 \right], \label{eq:pop_ipo_loss}
\end{align}
where $I(y, y''), I(y', y'')$ are unbiased estimators for $\cP(y \succ y''), \cP(y' \succ y'')$, respectively.
Then
\begin{align*}
    \bbE [\nabla_\theta \wh{\cL} (\theta; \samp, \mu)] = \nabla_\theta \cL_\ipo (\theta; \samp, \mu).
\end{align*}
\end{theorem}

\subsubsection{Remarks} \label{sec:ipo_remarks}

This result has significant implications, as it enables us to implement \eg as a \emph{\textbf{single-step optimization}} algorithm (e.g., $\theta^{(t+1)} = \theta^{(t)} - \eta G (\theta^{(t)})$).
It eliminates the need for \emph{\textbf{nested optimization}} involving inner loss minimization (e.g., $\theta^{(t+1)} = \arg\min_\theta \cL_{\textup{inner}} (\theta; \eta, \theta^{(t)})$) to perform policy iteration.
This valuable property is \emph{rare} in previous online RLHF works, where researchers typically use a small number of inner-layer gradient descent steps to \emph{approximately} compute $\wh{\theta}^{(t+1)}$.
Their theoretical frameworks usually fail to account for such approximation errors, further widening the gap between theory and practice.
As we will discuss in \Cref{sec:ipo_nlhf}, online mirror descent (OMD) and Nash-MD can also benefit from this online IPO formulation with implementations that more faithfully reflect the theory than performing gradient descent on $\cL_{\textup{inner}}$.
We believe this approach can be extended to many other algorithms for online RLHF.
\section{Experiments} \label{sec:exp}

For experiments, we compare our algorithm, \eg, with four baselines: Online IPO 1 (OMD), Online IPO 2, Nash-MD, and Nash-MD-PG \citep{munos2023nash}.
Implementation details of these baselines are provided in \Cref{sec:ipo_nlhf}.
For language model alignments, we also compare with MPO.
We exclude SPO, SPPO, and ONPO from comparison as they are not designed for the regularized game setting.
Since INPO is a minor modification of online mirror descent (OMD, i.e., Online IPO 1), we do not implement it separately.

\subsection{Numerical simulations}

We first report numerical simulation results on multi-armed bandits.

\para{Experiment setup.}
We examine four settings across two dimensions: $($\textbf{exact}, \textbf{empirical}$)\times($\textbf{tabular}, \textbf{neural}$)$.
The first dimension indicates whether we use estimation for the parameter update, while the second specifies whether we employ a tabular policy (with rigorous theoretical guarantees) or a neural network (used in practice for handling larger $\cY$s).

\para{Results.}
We present three experiments for \textbf{exact tabular} algorithms with different choices of $\beta$s in \Cref{fig:sim_tab_dual_gap_picked}.
Additional experiments and details are provided in \Cref{sec:more_numerical}.
For experiments with the same $\beta$, we use identical $\eta$ across all algorithms and the same mixture coefficient ($0.125$, according to \citet{munos2023nash}) for both Nash-MD and Nash-MD-PG.

\begin{figure}[t]
    \centering
    \includegraphics[width=\dimexpr0.98\linewidth\relax, height=\dimexpr0.98\textheight\relax, keepaspectratio]{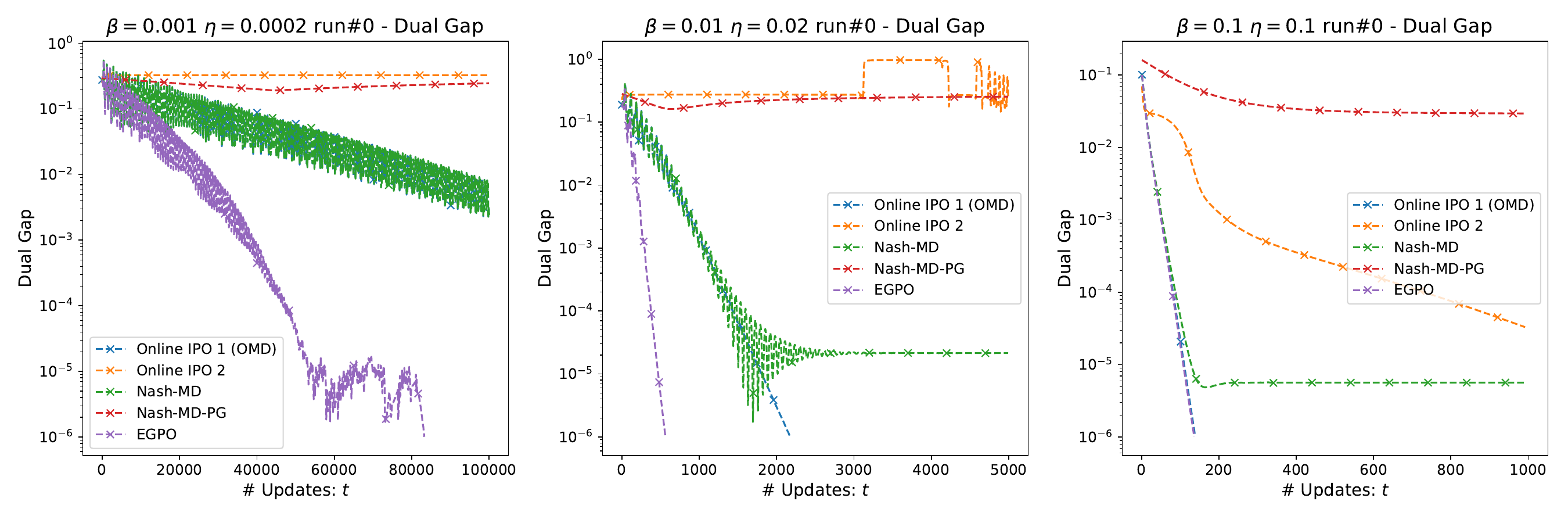}
    \vspace{-1em}
    \caption{Duality gap ($\dualgap_\beta$) of \textbf{exact tabular} algorithms with different $\beta$s.
    Values are cut off below $10^{-6}$ due to floating point precision.
    These figures are the $0$th experiments shown in \Cref{fig:sim_tab_dual_gap_beta_0.001,fig:sim_tab_dual_gap_beta_0.01,fig:sim_tab_dual_gap_beta_0.1}.}
    \label{fig:sim_tab_dual_gap_picked}
\end{figure}

\para{Remarks.}
From \Cref{fig:sim_tab_dual_gap_picked}, we observe the following:


$\bullet$ With identical learning rates, \eg consistently demonstrates among the fastest convergence rates, with its advantage over other baselines maximized at $\beta = 0.001$, the most challenging case as it most closely resembles the original matrix game.

$\bullet$ For large $\beta$, Online IPO 1 (OMD) performs similarly to \eg, suggesting that $\pi^{(t+1/2)}$ closely approximates $\pi^{(t+1)}$.


$\bullet$ Nash-MD converges linearly to a larger value, confirming Theorem~1 in \citet{munos2023nash}.
    Nash-MD-PG converges only when $\beta$ is large.
    These results demonstrate the advantage of our online IPO formulation over nested optimization.

\subsection{Language model alignments}
We provide a brief description of our experiments here with details in \Cref{sec:more_lm}.

\para{Experiment setup.}
We fine-tune a \texttt{gemma-2-2b-it} model \citep{team2024gemma} for \emph{sequence classification} on a mixture of widely-used open-source preference datasets as the ground truth preference $\cP$.
We emphasize that \textbf{SFT for $\cP$ does not constitute \emph{preference modeling}, but serves as the \emph{ground truth} preference.}
With sufficient resources, this model can be replaced with human annotators or LLMs.
We fine-tune another \texttt{gemma-2-2b-it} model for \emph{causal language modeling} on the Alpaca dataset \citep{alpaca} as both the reference policy $\pi_\tref$ and the initialization $\pi^{(0)}$.
We use the PKU-SafeRLHF dataset \citep{Ji2023BeaverTailsTI,ji2024pku} as our NLHF dataset.
For Nash-MD-PG, we use the implementation in the TRL library \citep{vonwerra2022trl};
for MPO, we use the official implementation;
and for all other algorithms, we implement custom trainers under the online IPO formulation.

\para{Approximating uniform sampling.}
Directly sampling from $\samp$, the uniform distribution over the response space, in an auto-regressive manner is impractical, as most sampled responses would be meaningless.
Given a prompt $x$, we constrain the response space $\cY (x)$ to be the subset containing only \emph{meaningful} responses.
To sample uniformly from this implicitly defined set, we generate responses using $\pi^{(t)}$ with \texttt{top\_k} $=10$ and \texttt{temperature} $=2$ as an approximation. 

\para{Results.}
We run each algorithm for $10$ epochs and compare each checkpoint $\pi_{\mathsf{ALG}}^{(k)}$ with the reference policy $\pi_\tref$ by querying the ground truth preference $\cP$.
Based on win-rates against $\pi_\tref$ (see \Cref{tab:res_safe_against_ref}), $\cP (\pi_{\mathsf{ALG}}^{(k)} \succ \pi_\tref)$, we select the top $2$ checkpoints for each algorithm and report their pairwise win-rates in \Cref{tab:res_safe}.
Generated text samples from models trained with different algorithms are presented in \Cref{sec:gen_results}.

\begin{table}[t]
\begin{center}
\resizebox{\textwidth}{!}{
\begin{tabular}{cc||c|cc|cc|cc|cc|cc|cc}
\toprule
ALG &  & \multirow{2}{*}{$\pi_\tref$} & \multicolumn{2}{c|}{\oipoone} & \multicolumn{2}{c|}{\oipotwo} & \multicolumn{2}{c|}{\nmd} & \multicolumn{2}{c|}{\nmdpg} & \multicolumn{2}{c|}{\mpo} & \multicolumn{2}{c}{\eg} \\
 & Ep & & 6 & 8 & 6 & 9 & 8 & 10 & 4 & 8 & 7 & 8 & 5 & 8 \\
\hhline{===============}
\multirow{2}{*}{\oipoone} & 6 & $\red{\boldsymbol{72.8\%}}$ & & & $\red{\boldsymbol{58.6\%}}$ & $\red{\boldsymbol{57.6\%}}$ & $47.7\%$ & $46.4\%$ & $\red{\boldsymbol{68.4\%}}$ & $\red{\boldsymbol{69.4\%}}$ & $45.2\%$ & $47.0\%$ & $42.6\%$ & $42.8\%$\\
 & 8 & $\red{\boldsymbol{71.8\%}}$ & & & $\red{\boldsymbol{58.9\%}}$ & $\red{\boldsymbol{58.7\%}}$ & $48.1\%$ & $47.0\%$ & $\red{\boldsymbol{68.2\%}}$ & $\red{\boldsymbol{68.0\%}}$ & $45.7\%$ & $47.2\%$ & $42.1\%$ & $43.6\%$\\
\hline
\multirow{2}{*}{\oipotwo} & 6 & $\red{\boldsymbol{66.8\%}}$ & $41.4\%$ & $41.1\%$ & & & $39.8\%$ & $38.5\%$ & $\red{\boldsymbol{62.3\%}}$ & $\red{\boldsymbol{61.3\%}}$ & $41.3\%$ & $42.8\%$ & $33.8\%$ & $35.2\%$\\
 & 9 & $\red{\boldsymbol{66.3\%}}$ & $42.4\%$ & $41.3\%$ & & & $38.5\%$ & $38.7\%$ & $\red{\boldsymbol{61.2\%}}$ & $\red{\boldsymbol{61.3\%}}$ & $40.8\%$ & $42.7\%$ & $34.2\%$ & $33.8\%$\\
\hline
\multirow{2}{*}{\nmd} & 8 & $\red{\boldsymbol{72.8\%}}$ & $\red{\boldsymbol{52.3\%}}$ & $\red{\boldsymbol{51.9\%}}$ & $\red{\boldsymbol{60.2\%}}$ & $\red{\boldsymbol{61.5\%}}$ & & & $\red{\boldsymbol{70.0\%}}$ & $\red{\boldsymbol{71.1\%}}$ & $46.4\%$ & $48.3\%$ & $44.0\%$ & $46.7\%$\\
 & 10 & $\red{\boldsymbol{72.9\%}}$ & $\red{\boldsymbol{53.6\%}}$ & $\red{\boldsymbol{53.0\%}}$ & $\red{\boldsymbol{61.5\%}}$ & $\red{\boldsymbol{61.3\%}}$ & & & $\red{\boldsymbol{70.6\%}}$ & $\red{\boldsymbol{71.2\%}}$ & $47.3\%$ & $49.2\%$ & $44.6\%$ & $45.8\%$\\
\hline
\multirow{2}{*}{\nmdpg} & 4 & $\red{\boldsymbol{55.2\%}}$ & $31.6\%$ & $31.8\%$ & $37.7\%$ & $38.8\%$ & $30.0\%$ & $29.4\%$ & & & $31.5\%$ & $33.2\%$ & $26.2\%$ & $26.4\%$\\
 & 8 & $\red{\boldsymbol{55.1\%}}$ & $30.6\%$ & $32.0\%$ & $38.7\%$ & $38.7\%$ & $28.9\%$ & $28.8\%$ & & & $31.1\%$ & $32.2\%$ & $26.2\%$ & $25.8\%$\\
\hline
\multirow{2}{*}{\mpo} & 7 & $\red{\boldsymbol{71.9\%}}$ & $\red{\boldsymbol{54.8\%}}$ & $\red{\boldsymbol{54.3\%}}$ & $\red{\boldsymbol{58.7\%}}$ & $\red{\boldsymbol{59.2\%}}$ & $\red{\boldsymbol{53.6\%}}$ & $\red{\boldsymbol{52.7\%}}$ & $\red{\boldsymbol{68.5\%}}$ & $\red{\boldsymbol{68.9\%}}$ & & & $49.4\%$ & $47.9\%$\\
 & 8 & $\red{\boldsymbol{70.2\%}}$ & $\red{\boldsymbol{53.0\%}}$ & $\red{\boldsymbol{52.8\%}}$ & $\red{\boldsymbol{57.2\%}}$ & $\red{\boldsymbol{57.3\%}}$ & $\red{\boldsymbol{51.7\%}}$ & $\red{\boldsymbol{50.8\%}}$ & $\red{\boldsymbol{66.8\%}}$ & $\red{\boldsymbol{67.8\%}}$ & & & $47.2\%$ & $46.9\%$\\
\hline
\multirow{2}{*}{\eg} & 5 & $\red{\boldsymbol{76.9\%}}$ & $\red{\boldsymbol{57.4\%}}$ & $\red{\boldsymbol{57.9\%}}$ & $\red{\boldsymbol{66.2\%}}$ & $\red{\boldsymbol{65.8\%}}$ & $\red{\boldsymbol{56.0\%}}$ & $\red{\boldsymbol{55.4\%}}$ & $\red{\boldsymbol{73.8\%}}$ & $\red{\boldsymbol{73.8\%}}$ & $\red{\boldsymbol{50.6\%}}$ & $\red{\boldsymbol{52.8\%}}$ & &\\
 & 8 & $\red{\boldsymbol{77.4\%}}$ & $\red{\boldsymbol{57.2\%}}$ & $\red{\boldsymbol{56.4\%}}$ & $\red{\boldsymbol{64.8\%}}$ & $\red{\boldsymbol{66.2\%}}$ & $\red{\boldsymbol{53.3\%}}$ & $\red{\boldsymbol{54.2\%}}$ & $\red{\boldsymbol{73.6\%}}$ & $\red{\boldsymbol{74.2\%}}$ & $\red{\boldsymbol{52.1\%}}$ & $\red{\boldsymbol{53.1\%}}$ & &\\
\bottomrule
\end{tabular}
}

\end{center}
\caption{Pairwise win-rates evaluated by the ground truth preference on PKU-SafeRLHF. Each number is the win-rate of the \textbf{row} model against the \textbf{column} model. \textbf{Abbreviations:} ``Ep'' stands for the epoch number; ``\oipoone'' stands for ``Online IPO 1 (OMD)''; ``\oipotwo'' stands for ``Online IPO 2''; ``\nmd'' stands for ``Nash-MD''; ``\nmdpg'' stands for ``Nash-MD-PG''; ``\mpo'' stands for ``magnetic preference optimization''; ``\eg'' stands for ``Extragradient preference optimization''. Win-rates larger than $50\%$ are boldfaced red texts.}
\label{tab:res_safe}
\end{table}

\para{Remarks.}
From \Cref{tab:res_safe,tab:res_safe_against_ref}, we observe the following:

$\bullet$ \eg outperforms all other algorithms, both in win-rates against the reference policy and in pairwise comparisons.

$\bullet$ The comparison between Nash-MD and Nash-MD-PG further confirms the advantage of our online IPO formulation over approximate nested optimization.

$\bullet$ The ground truth preference $\cP$ demonstrates non-transitive behavior: while MPO achieves lower win-rates against the reference policy than Nash-MD, it consistently beats Nash-MD in direct pairwise comparisons.

From the LLM \emph{generation} experimental results in \Cref{sec:gen_results}, we can see that \eg's responses contain both the argument that the entity in question is harmful and an alternative safe solution.
\section{Conclusion}

We presented \eg, which achieves last-iterate linear convergence to the Nash equilibrium (NE) of KL-regularized preference games without requiring nested optimization.
\eg also has practical advantages for language model alignment with non-transitive human preferences.
Our empirical results confirm \eg's superior performance over baselines.

We acknowledge several limitations of our study that can
inspire future research.
First, like previous RLHF works, our algorithm designs are based on \emph{tabular softmax parametrization}; a natural extension would be to study \emph{log-linear parametrization} and \emph{function approximation}. 
Second, \eg effectively uses two consecutive iterations to update the policy once, highlighting the need for novel designs that reduce sample complexity while maintaining last-iterate linear convergence.
Third, our linear convergence remains slower than the \emph{quadratic} convergence established by \citet{shi2025the} for DPO, which was achieved using a non-trivial sampling distribution.
This raises the question of whether faster convergence to NE is possible using similar approaches.

\section*{Acknowledgement}
SSD acknowledges the support of NSF DMS 2134106, NSF CCF 2212261, NSF IIS 2143493, NSF IIS 2229881, Alfred P. Sloan Research Fellowship, and Schmidt Sciences AI 2050 Fellowship.
RZ and MF acknowledge the support of NSF TRIPODS II DMS-2023166. The work of MF was supported in part by awards NSF CCF 2212261 and NSF CCF 2312775.

\bibliography{ref}
\bibliographystyle{COLM2025/colm2025_conference}

\appendix

\section{Additional related works} \label{sec:additional_rel}

\para{Reinforcement learning from human feedback (RLHF) with a reward function.}
RLHF \citep{Christiano2017DeepRL,Ziegler2019FineTuningLM,stiennon2020learning,ouyang2022training,Bai2022TrainingAH} evolved from preference-based RL \citep{Wirth2013PreferenceBasedRL,Wirth2017ASO,9945333}, where agents learn scalar rewards from preference feedback and optimize policies using these learned rewards.
The key assumption underlying reward learning is that preferences are parameterized by a reward function.
\citet{Zhu2023PrincipledRL} formulate RLHF as contextual bandits and prove the convergence of the maximum likelihood estimator.
\citet{xie2024exploratory} examine the online exploration problem from the perspective of KL-regularized Markov decision processes (MDPs) and 
give provable guarantees (in sample complexity) for an exploration bonus.
\citet{Liu2024ProvablyMO} investigate the overoptimization issue and prove a finite-sample suboptimality gap.

\para{Bypassing reward learning in RLHF.}
The two-stage formulation in RLHF is both unstable and inefficient.
To address this issue, direct preference optimization (DPO, \citet{rafailov2023direct}) utilizes the closed-form solution of the KL-regularized RLHF objective to directly learn the policy and is further extended to the MDP setting \citep{rafailov2024rqlanguagemodel}.
DPO has spawned many variants, such as $\Psi$-PO \citep{Azar2023AGT}, RPO \citep{Liu2024ProvablyMO}, CPO \citep{Xu2024ContrastivePO}, SimPO \citep{Meng2024SimPOSP}, DQO \citep{liu2024enhancing}, and $\chi$PO \citep{huang2024correcting}.
Other works have proposed generalizations \citep{wang2023beyond,tang2024generalized}.
While vanilla DPO is inherently offline, several studies have designed and analyzed online or iterative DPO algorithms: \citet{xiong2024iterative,song2024importanceonlinedataunderstanding,xie2024exploratory,guo2024direct,Tajwar2024PreferenceFO,Ding2024SAILSE,dong2024rlhf,shi2025the,feng2025pilaf,chen2025avoidingmathbfexprmaxscalingrlhf}.

\section{Additional concepts of RLHF} \label{sec:additional_prelim}

\subsection{Standard bandit learning}

We begin with basic concepts of bandit learning, which forms the foundation for RLHF.

\para{Multi-armed bandits and contextual bandits.}
A multi-armed bandit has an arm (action) space $\cY$ and a reward function $r: \cY \to [0, 1]$.
A contextual bandit has a context space $\cX$, an arm space $\cY$, and a reward function $r: \cX \times \cY \to [0, 1]$.
In this work, the user prompt serves as a context, and the agent response as an arm.
To simplify notation, our results are stated in the \emph{multi-armed bandits} framework.
The statements and proofs can be easily extended to \emph{contextual bandits}.
Thus, we omit the prompts (contexts) and slightly abuse notation throughout the paper.

\subsection{Reinforcement learning from human feedback (RLHF)}

We now formally define the RLHF problems.

\para{RLHF with a Bradley-Terry (BT) preference.}
Given an implicit reward oracle $r: \cX \times \cY \to [0, 1]$, \cite{BTmodel} assume that human preference $\cP: \cX \times \cY \times \cY \rightarrow\Delta (\{0, 1\})$ satisfies:
\begin{align*}
    \cP (y_1 \succ y_2 | x) = \sigma \left(r (x, y_1) - r (x, y_2)\right),
    \quad \textup{where} \quad \sigma(t) = \frac{1}{1 + \exp(-t)}.
\end{align*}
This means that conditioned on prompt $x$, response $y_1$ is favored over $y_2$ with probability $\cP (y_1 \succ y_2 | x)$ by human annotators.
A human preference dataset $\cD = \{ (x^{(i)}, y_w^{(i)}, y_l^{(i)}) \}_{i=1}^N$ indicates that in the $i^{\text{th}}$ sample, $y_w^{(i)} \succ y_l^{(i)}$ conditioned on $x^{(i)}$.
The reward function $r: \cX \times \cY \rightarrow \mathbb R$ is learned with parameter $\phi$ using a negative log-likelihood loss:
\begin{align}
    \cL_r (\phi) = -\frac{1}{N} \sum_{i = 1}^N \log \sigma \left(r_{\phi} (x^{(i)}, y_w^{(i)}) - r_{\phi} (x^{(i)}, y_l^{(i)})\right). \label{eq:reward_loss}
\end{align}
Based on a reference policy $\piref$, the goal of RLHF is to maximize the obtained rewards with a KL-divergence penalty:
\begin{align}
    \pi_\phi^\star=\arg\max_{\pi\in\Pi} \bbE_{x\sim \rho(\cX)} \left[ \bbE_{y\sim \pi(\cdot | x)}r_\phi(x,y)-\beta\KL(\pi (\cdot | x) || \piref (\cdot | x)) \right],\label{eq:rlhf_obj}
\end{align}
where $\rho (\cX)$ is a probability distribution over $\cX$, and $\beta\in\mathbb R_+$ is the regularization coefficient.
Additionally, under \emph{tabular softmax parametrization}, we can derive the closed-form solution (Equation (4) in \cite{rafailov2023direct}):
\begin{align*}
    \pi_\phi^\star(y\vert x)=\frac{1}{Z_\phi (x)}\piref(y\vert x)\exp\left(\frac{1}{\beta} r_\phi(x,y)\right)\;,\;\forall x\in \cX ,y\in \cY,
\end{align*}
where $Z_\phi (x) = \sum_{y \in \cY} \piref(y\vert x)\exp\left(\frac{1}{\beta} r_\phi(x,y)\right)$ is the partition function.
Equivalently, the parameter $\theta_\phi^\star$ of the policy $\pi_\phi^\star$ satisfies
\begin{align}
    \theta_\phi^\star = \theta_\tref + \frac{r_\phi}{\beta}. \label{eq:closed_form}
\end{align}

\section{Technical lemmas}

\begin{lemma} [Pinsker's inequality] \label{lem:pinsker}
For any two probability distributions $p$ and $q$ defined on the same set,
\begin{align*}
    \norm{p - q}_1 \le \sqrt{2 \KL (p || q)}.
\end{align*}
\end{lemma}

\begin{definition} \label{def:subG}
Let $X$ be a random variable.
We say $X$ is sub-Gaussian with a variance proxy $\sigma^2$ if for any $t \ge 0$,
\begin{align*}
    \bbP [\abs{X} > t] \le 2 \exp \left(- \frac{t^2}{2 \sigma^2} \right),
\end{align*}
and we denote as $X \sim \subG (\sigma^2)$.
\end{definition}

\begin{lemma} [Lemma\,1.4 in \citet{mit2015highdim}] \label{lem:sub_gaussian_moment}
Let $X \sim \subG (\sigma^2)$, then for any positive integer $k \ge 1$,
\begin{align*}
    \bbE [\abs{X}^k] \le (2 \sigma^2)^{k/2} k \Gamma (k/2).
\end{align*}
\end{lemma}

\begin{lemma} \label{lem:sub_gaussian_vec_inf_norm_square}
Let $X$ be a $d$-dimensional random vector such that for $1 \le i \le d$, all $X_i$s are independent, and $X_i \sim \subG (\sigma^2)$, then
\begin{align*}
    \bbE [\norm{X}_\infty^2] \le 4 \sigma^2 \log (3 d).
\end{align*}
\end{lemma}

\begin{proof} [Proof of \Cref{lem:sub_gaussian_vec_inf_norm_square}]
We first derive an upper-bound for the moment generating function of each coordinate.
By dominated convergence theorem, when $0 < \lambda < 1 / (2 \sigma^2)$,
\begin{align*}
    \bbE [\exp(\lambda X_i^2)]
    &\le 1 + \sum_{k = 1}^{\infty} \frac{\lambda^k \bbE[X_i^{2k}]}{k!} \\
    &\mylei 1 + \sum_{k = 1}^{\infty} \frac{\lambda^k (2 \sigma^2)^k \cdot 2k \cdot \Gamma(k)}{k!} \\
    &= 1 + 2 \sum_{k = 1}^{\infty} (2 \sigma^2 \lambda)^k \\
    &= \frac{1 + 2 \sigma^2 \lambda}{1 - 2 \sigma^2 \lambda},
\end{align*}
where (i) is by \Cref{lem:sub_gaussian_moment}.
Then,
\begin{align*}
    \exp(\lambda \bbE [\norm{X}_\infty^2])
    &\le \bbE [\exp (\lambda \norm{X}_\infty^2 )] \\
    &= \bbE [\exp (\lambda \max_i X_i^2 )] \\
    &= \bbE [\max_i \exp (\lambda X_i^2 )] \\
    &\le \bbE \left[ \sum_{i=1}^d \exp (\lambda X_i^2 ) \right] \\
    &= \sum_{i=1}^d \bbE [ \exp (\lambda X_i^2 )] \\
    &\le d \frac{1 + 2 \sigma^2 \lambda}{1 - 2 \sigma^2 \lambda}.
\end{align*}
So
\begin{align*}
    \bbE [\norm{X}_\infty^2] \le \frac{1}{\lambda} \left(\log d + \log \frac{1 + 2 \sigma^2 \lambda}{1 - 2 \sigma^2 \lambda}\right).
\end{align*}
Taking $\lambda = 1 / (4 \sigma^2)$ gives the final result.
\end{proof}

\section{Proofs}

\subsection{Convergence of \eg} \label{sec:proof_extragrad}

\begin{theorem}[Full statement of \Cref{thm:extragrad_short}] \label{thm:extragrad}
For any initialization $\theta^{(0)}$, following the update rules defined by \Cref{eq:extragrad_est,eq:extragrad_upd} and setting $\eta \le \frac{1}{\beta + 3}$, we have that for any $T \ge 1$,
\begin{align}
\bbE[\KL (\pi_\beta^\star || \pi^{(T)})] &\le \KL (\pi_\beta^\star || \pi^{(0)}) (1 - \eta \beta)^T + \frac{4 \sigma^2 \log (3 \abs{\cY})}{\beta}, \label{eq:extragrad_kl} \\
\bbE[\KL (\pi^{(T)} || \pi_\beta^\star)] &\le \frac{2 \KL (\pi_\beta^\star || \pi^{(0)})}{\eta \beta} (1 - \eta \beta)^T + \frac{8 \sigma^2 \log (3 \abs{\cY})}{\eta \beta^2}, \label{eq:extragrad_rev_kl} \\
\bbE[\dualgap_\beta (\pi^{(T)})] &\le \left( \frac{2}{\beta} + \frac{4}{\eta} \right) \KL (\pi_\beta^\star || \pi^{(0)}) (1 - \eta \beta)^T + \left( \frac{8}{\beta^2} + \frac{16}{\eta \beta} \right) \sigma^2 \log (3 \abs{\cY}), \label{eq:extragrad_dualgap}
\end{align}
and for any $T \ge 0$,
\begin{align}
\bbE[\KL (\pi_\beta^\star || \pi^{(T+1/2)})] &\le 2 \KL (\pi_\beta^\star || \pi^{(0)}) (1 - \eta \beta)^{T} + \frac{8 \sigma^2 \log (3 \abs{\cY})}{\beta}, \label{eq:extragrad_est_kl} \\
\bbE[\KL (\pi^{(T+1/2)} || \pi_\beta^\star)] &\le \frac{\KL (\pi_\beta^\star || \pi^{(0)})}{\eta \beta} (1 - \eta \beta)^{T + 1} + \frac{4 \sigma^2 \log (3 \abs{\cY})}{\eta \beta^2}, \label{eq:extragrad_est_rev_kl} \\
\bbE[\dualgap_\beta (\pi^{(T+1/2)})] &\le \left( \frac{4}{\beta} + \frac{2}{\eta} \right) \KL (\pi_\beta^\star || \pi^{(0)}) (1 - \eta \beta)^T + \left( \frac{16}{\beta^2} + \frac{8}{\eta \beta} \right) \sigma^2 \log (3 \abs{\cY}). \label{eq:extragrad_est_dualgap}
\end{align}
Under the \textbf{exact update} scheme where $\sigma^2 = 0$, all the expectations are removed.
\end{theorem}

The following lemmas will be extensively used throughout the proof.

\begin{lemma} \label{lem:diff_self_quad_form}
For $p, q \in \Delta^\cY$, we have that
\begin{align*}
    (p - q)^\top \cP (p - q) = 0.
\end{align*}
\end{lemma}

\begin{proof} [Proof of \Cref{lem:diff_self_quad_form}]
Direct computation gives
\begin{align*}
    0 = (p - q)^\top \II (p - q)
    = (p - q)^\top (\cP + \cP^\top) (p - q)
    = 2 (p - q)^\top \cP (p - q).
\end{align*}
\end{proof}

\begin{lemma} \label{lem:diff_quad_form}
For $p_1, q_1, p_2, q_2 \in \Delta^\cY$ and $\xi > 0$, we have that
\begin{align*}
    (p_1 - q_1)^\top \cP (p_2 - q_2) \le \xi \min\{\KL (p_1 || q_1), \KL (q_1 || p_1)\} + \frac{1}{\xi} \min\{ \KL (p_2 || q_2), \KL (q_2 || p_2) \}.
\end{align*}
\end{lemma}

\begin{proof} [Proof of \Cref{lem:diff_quad_form}]
Direct computation gives
\begin{align*}
    (p_1 - q_1)^\top \cP (p_2 - q_2)
    &= \sum_{y, y'} (p_1 - q_1)_y \cP_{y, y'} (p_2 - q_2)_{y'} \\
    & \le \max_{y, y'} \abs{\cP_{y, y'}} \cdot \norm{p_1 - q_1}_1 \norm{p_2 - q_2}_1 \\
    & \mylei \frac{\xi}{2} \norm{p_1 - q_1}_1^2 + \frac{1}{2 \xi} \norm{p_2 - q_2}_1^2 \\
    & \myleii \xi \min\{\KL (p_1 || q_1), \KL (q_1 || p_1)\} + \frac{1}{\xi} \min\{ \KL (p_2 || q_2), \KL (q_2 || p_2) \},
\end{align*}
where (i) is by $\max_{y, y'} \abs{\cP_{y, y'}} \le 1$;
(ii) is by Pinsker's inequality (\Cref{lem:pinsker}).
\end{proof}

\subsubsection{Bounding KL divergence}

We will use the following relations frequently:
\begin{align*}
    \langle \theta_1, \pi_{\theta_2} - \pi_{\theta_3} \rangle = \langle \log \pi_{\theta_1}, \pi_{\theta_2} - \pi_{\theta_3} \rangle.
\end{align*}

\para{For \Cref{eq:extragrad_kl}.}
Since $\theta_\beta^\star$ is a solution of \Cref{eq:main_equation}, we write
\begin{align*}
    \theta_\tref = \theta_\beta^\star - \frac{\cP \pi_\beta^\star}{\beta}.
\end{align*}
Plugging into \Cref{eq:extragrad_upd}, we have
\begin{align}
    \theta^{(t+1)} - (1 - \eta \beta) \theta^{(t)} - \eta \beta \theta_\beta^\star &= \eta \cP (\pi^{(t+1/2)} - \pi_\beta^\star) + \eta \epsilon^{(t+1/2)}, \notag \\
    \Rightarrow \langle \theta^{(t+1)}  - (1 - \eta \beta) \theta^{(t)} - \eta \beta \theta_\beta^\star, \pi^{(t+1/2)} - \pi_\beta^\star \rangle &\myeqi \eta \langle \epsilon^{(t+1/2)}, \pi^{(t+1/2)} - \pi_\beta^\star \rangle, \label{eq:core1}
\end{align}
where (i) is by \Cref{lem:diff_self_quad_form}.
Since $\pi_\beta^\star$ is a fixed policy, and $\bbE [\epsilon^{(t+1/2)} | \pi^{(t+1/2)}] = \boldsymbol{0}$, we have that
\begin{align*}
    \bbE[\langle \epsilon^{(t+1/2)}, \pi_\beta^\star \rangle] = 0 = \bbE[\langle \epsilon^{(t+1/2)}, \pi^{(t+1/2)} \rangle].
\end{align*}
Taking expectation,
\begin{align*}
    \bbE [\langle \log \pi^{(t+1)}  - (1 - \eta \beta) \log \pi^{(t)} - \eta \beta \log \pi_\beta^\star , \pi^{(t+1/2)} - \pi_\beta^\star \rangle] = 0.
\end{align*}
We have
\begin{align*}
    \langle \log \pi^{(t+1)} - (1 - \eta \beta) \log \pi^{(t)} - \eta \beta \log \pi_\beta^\star, - \pi_\beta^\star \rangle
    = -(1 - \eta \beta) \KL (\pi_\beta^\star || \pi^{(t)}) + \KL (\pi_\beta^\star || \pi^{(t+1)}),
\end{align*}
and
\begin{align*}
    &\langle \log \pi^{(t+1)} - (1 - \eta \beta) \log \pi^{(t)}  - \eta \beta \log \pi_\beta^\star , \pi^{(t+1/2)} \rangle \\
    &= \langle \log \pi^{(t+1/2)} - (1 - \eta \beta) \log \pi^{(t)}  - \eta \beta \log \pi_\beta^\star , \pi^{(t+1/2)} \rangle + \langle \log \pi^{(t+1/2)} - \log \pi^{(t+1)}, \pi^{(t+1)} \rangle \\
    &\quad - \langle \log \pi^{(t+1/2)} - \log \pi^{(t+1)}, \pi^{(t+1/2)} - \pi^{(t+1)} \rangle \\
    &= (1 - \eta \beta) \KL (\pi^{(t+1/2)} || \pi^{(t)}) + \eta \beta \KL (\pi^{(t+1/2)} || \pi_\beta^\star) + \KL (\pi^{(t+1)} || \pi^{(t+1/2)}) \\
    &\quad + \langle \theta^{(t+1/2)} - \theta^{(t+1)}, \pi^{(t+1)} - \pi^{(t+1/2)} \rangle.
\end{align*}
By \Cref{eq:extragrad_est,eq:extragrad_upd},
\begin{align*}
    &\langle \theta^{(t+1/2)} - \theta^{(t+1)}, \pi^{(t+1)} - \pi^{(t+1/2)} \rangle \\
    &= \eta (\pi^{(t+1)} - \pi^{(t+1/2)})^\top \cP (\pi^{(t)} - \pi^{(t+1/2)}) + \eta \langle \epsilon^{(t)} - \epsilon^{(t+1/2)}, \pi^{(t+1)} - \pi^{(t+1/2)} \rangle \\
    &\mylei \eta ( \KL( \pi^{(t+1)} || \pi^{(t+1/2)}) + \KL(\pi^{(t+1/2)} || \pi^{(t)}) + \langle \epsilon^{(t)} - \epsilon^{(t+1/2)}, \pi^{(t+1)} - \pi^{(t+1/2)} \rangle ),
\end{align*}
where (i) is by \Cref{lem:diff_quad_form} with $\xi = 1$.
Next we bound $\langle \epsilon^{(t)} - \epsilon^{(t+1/2)}, \pi^{(t+1/2)} - \pi^{(t+1)} \rangle$.
\begin{align*}
    \langle \epsilon^{(t)} - \epsilon^{(t+1/2)}, \pi^{(t+1)} - \pi^{(t+1/2)} \rangle
    &\le \norm{\epsilon^{(t)} - \epsilon^{(t+1/2)}}_\infty \norm{\pi^{(t+1)} - \pi^{(t+1/2)}}_1 \\
    &\le \frac{1}{2} \norm{\epsilon^{(t)} - \epsilon^{(t+1/2)}}_\infty^2 + \frac{1}{2} \norm{\pi^{(t+1)} - \pi^{(t+1/2)}}_1^2 \\
    &\mylei \frac{1}{2} \norm{\epsilon^{(t)} - \epsilon^{(t+1/2)}}_\infty^2 + \KL (\pi^{(t+1)} || \pi^{(t+1/2)}),
\end{align*}
where (i) is by Pinsker's inequality (\Cref{lem:pinsker}).
By \Cref{lem:sub_gaussian_vec_inf_norm_square},
\begin{align*}
    \bbE \left[ \norm{\epsilon^{(t)} - \epsilon^{(t+1/2)}}_\infty^2 \right] \le 8 \sigma^2 \log (3 \abs{\cY}).
\end{align*}
Putting these terms together,
\begin{align}
    \bbE [\KL (\pi_\beta^\star || \pi^{(t+1)})]
    &\le (1 - \eta \beta) \bbE [\KL (\pi_\beta^\star || \pi^{(t)})] - (1 - \eta \beta - \eta) \bbE [\KL (\pi^{(t+1/2)} || \pi^{(t)})]  \notag \\
    &\quad - \eta \beta \bbE [\KL (\pi^{(t+1/2)} || \pi_\beta^\star)] - (1 - 2 \eta) \bbE [\KL (\pi^{(t+1)} || \pi^{(t+1/2)})] \notag \\
    &\quad + 4 \eta \sigma^2 \log (3 \abs{\cY}). \label{eq:extragrad_recursion}
\end{align}
By choosing $\eta \le \min\{\frac{1}{\beta + 1}, \frac{1}{2}\}$, we have that
\begin{align*}
    \bbE [\KL (\pi_\beta^\star || \pi^{(t+1)})]
    &\le (1 - \eta \beta) \bbE [\KL (\pi_\beta^\star || \pi^{(t)})] + 4 \eta \sigma^2 \log (3 \abs{\cY}) \\
    &\le (1 - \eta \beta)^{t+1} \KL (\pi_\beta^\star || \pi^{(0)}) + \frac{4 \sigma^2 \log (3 \abs{\cY})}{\beta}.
\end{align*}

\para{For \Cref{eq:extragrad_est_kl}.}

We have
\begin{align*}
    &\KL (\pi_\beta^\star || \pi^{(t+1/2)}) \\
    &= \langle \log \pi_\beta^\star - \log \pi^{(t+1)}, \pi_\beta^\star \rangle - \langle \log \pi^{(t+1/2)} - \log \pi^{(t+1)}, \pi^{(t+1/2)} \rangle \\
    &\quad - \langle \log \pi^{(t+1/2)} - \log \pi^{(t+1)}, \pi_\beta^\star - \pi^{(t+1/2)} \rangle \\
    &= \KL (\pi_\beta^\star || \pi^{(t+1)}) - \KL (\pi^{(t+1/2)} || \pi^{(t+1)}) + \langle \theta^{(t+1/2)} - \theta^{(t+1)}, \pi^{(t+1/2)} - \pi_\beta^\star \rangle \\
    &\myeqi \KL (\pi_\beta^\star || \pi^{(t+1)}) - \KL (\pi^{(t+1/2)} || \pi^{(t+1)}) + \eta (\pi^{(t+1/2)} - \pi_\beta^\star)^\top \cP (\pi^{(t)} - \pi^{(t+1/2)}) \\
    &\quad + \eta \langle \epsilon^{(t)} - \epsilon^{(t+1/2)}, \pi^{(t+1/2)} - \pi_\beta^\star \rangle\\
    &\myleii \KL (\pi_\beta^\star || \pi^{(t+1)}) + \eta \KL (\pi_\beta^\star || \pi^{(t+1/2)}) + \eta \KL (\pi^{(t+1/2)} || \pi^{(t)}) \\
    &\quad + \eta \langle \epsilon^{(t)} - \epsilon^{(t+1/2)}, \pi^{(t+1/2)} - \pi_\beta^\star \rangle,
\end{align*}
where (i) is by \Cref{eq:extragrad_est,eq:extragrad_upd};
(ii) is by \Cref{lem:diff_quad_form} with $\xi = 1$.
We know that $\bbE [\langle \epsilon^{(t)} - \epsilon^{(t+1/2)}, \pi_\beta^\star \rangle] = \bbE [\langle \epsilon^{(t+1/2)}, \pi^{(t+1/2)} \rangle] = \bbE [\langle \epsilon^{(t)}, \pi^{(t)} \rangle] = 0$.
Thus,
\begin{align*}
    \bbE [\langle \epsilon^{(t)} - \epsilon^{(t+1/2)}, \pi_\beta^\star - \pi^{(t+1/2)} \rangle] 
    &= \bbE [\langle \epsilon^{(t)} , \pi^{(t)} - \pi^{(t+1/2)} \rangle] \\
    &\le \frac{1}{2} \bbE \left[ \norm{\epsilon^{(t)}}_\infty^2 \right] + \bbE[\KL (\pi^{(t+1/2)} || \pi^{(t)})] \\
    &\mylei 2 \sigma^2 \log (3 \abs{\cY}) + \bbE[\KL (\pi^{(t+1/2)} || \pi^{(t)})],
\end{align*}
where (i) is by \Cref{lem:sub_gaussian_vec_inf_norm_square}.
Hence,
\begin{align*}
    &\bbE[\KL (\pi_\beta^\star || \pi^{(t+1/2)})]\\
    &\le \frac{\bbE [\KL (\pi_\beta^\star || \pi^{(t+1)})] + 2 \eta \bbE [\KL (\pi^{(t+1/2)} || \pi^{(t)})] + 2 \eta \sigma^2 \log (3 \abs{\cY}) }{1 - \eta} \\
    &\mylei \frac{(1 - \eta \beta) \bbE [\KL (\pi_\beta^\star || \pi^{(t)})] - (1 - \eta \beta - 3 \eta) \bbE [\KL (\pi^{(t+1/2)} || \pi^{(t)})] + 6 \eta \sigma^2 \log (3 \abs{\cY}) }{1 - \eta} \\
    &\myleii \frac{(1 - \eta \beta)^{t+1} \KL (\pi_\beta^\star || \pi^{(0)}) + (\frac{4}{\beta} + 2 \eta) \sigma^2 \log (3 \abs{\cY}) }{1 - \eta},
\end{align*}
where (i) is by \Cref{eq:extragrad_recursion};
(ii) is by choosing $\eta \le \frac{1}{\beta + 3}$ and \Cref{eq:extragrad_kl}.
\Cref{eq:extragrad_est_kl} holds because when $\eta \le \frac{1}{\beta + 3}$, we have $1 - \eta \beta \le 2 (1 - \eta)$, and $\frac{4}{\beta} + 2 \eta \le \frac{8}{\beta} (1 - \eta)$.

\para{For \Cref{eq:extragrad_rev_kl}.}
By \Cref{eq:extragrad_upd},
\begin{align}
    &\langle \theta^{(t+1)} - (1 - \eta \beta) \theta^{(t)} - \eta \beta \theta_\beta^\star, \pi^{(t+1)} - \pi_\beta^\star \rangle \notag \\
    &= \eta (\pi^{(t+1)} - \pi_\beta^\star)^\top \cP (\pi^{(t+1/2)} - \pi_\beta^\star) + \eta \langle \epsilon^{(t+1/2)}, \pi^{(t+1)} - \pi_\beta^\star \rangle \label{eq:core2} \\
    &\mylei 2 \eta \KL (\pi_\beta^\star || \pi^{(t+1)}) + \eta  \KL (\pi_\beta^\star || \pi^{(t+1/2)}) + \eta \norm{\epsilon^{(t+1/2)}}_\infty^2 \notag
\end{align}
where (i) is by \Cref{lem:diff_quad_form} with $\xi = 1$.
At the same time,
\begin{align*}
    & \langle \theta^{(t+1)} - (1 - \eta \beta) \theta^{(t)} - \eta \beta \theta_\beta^\star, \pi^{(t+1)} - \pi_\beta^\star \rangle \\
    &=(1 - \eta \beta) \KL (\pi^{(t+1)} || \pi^{(t)}) + \eta \beta \KL (\pi^{(t+1)} || \pi_\beta^\star) + \KL (\pi_\beta^\star || \pi^{(t+1)}) - (1 - \eta \beta) \KL (\pi_\beta^\star || \pi^{(t)}).
\end{align*}
So
\begin{align*}
    \bbE [\KL (\pi^{(t+1)} || \pi_\beta^\star)]
    &\mylei \frac{(2 \eta - 1) \bbE [\KL (\pi_\beta^\star || \pi^{(t+1)})]}{\eta \beta} \\
    &\quad + \frac{\eta \bbE [\KL (\pi_\beta^\star || \pi^{(t+1/2)})] + (1 - \eta \beta) \bbE [\KL (\pi_\beta^\star || \pi^{(t)})] + 4 \eta \sigma^2 \log (3 \abs{\cY})}{\eta \beta} \\
    &\myleii \frac{(1 - \eta \beta + 2 \eta) [ (1 - \eta \beta)^t \KL (\pi_\beta^\star || \pi^{(0)}) + \frac{4}{\beta} \sigma^2 \log (3 \abs{\cY})] + 4 \eta \sigma^2 \log (3 \abs{\cY})}{\eta \beta} \\
    &\myleiii \frac{2 (1 - \eta \beta)^t \KL (\pi_\beta^\star || \pi^{(0)})}{\eta \beta} + \frac{8 \sigma^2 \log (3 \abs{\cY})}{\eta \beta^2},
\end{align*}
where (i) is by \Cref{lem:sub_gaussian_vec_inf_norm_square};
(ii) is by choosing $\eta \le \frac{1}{2}$ and \Cref{eq:extragrad_kl,eq:extragrad_est_kl};
(iii) is by choosing $\eta \le \frac{1}{\beta + 3}$.

\para{For \Cref{eq:extragrad_est_rev_kl}.}
From \Cref{eq:extragrad_recursion,eq:extragrad_kl}, we have that
\begin{align*}
    \bbE[\KL (\pi^{(t+1/2)} || \pi_\beta^\star)] &\le \frac{(1 - \eta \beta) \bbE[\KL (\pi_\beta^\star || \pi^{(t)})] + 4 \eta \sigma^2 \log (3 \abs{\cY})}{\eta \beta}\\
    &\le \frac{(1 - \eta \beta)^{t + 1} \KL (\pi_\beta^\star || \pi^{(0)})}{\eta \beta} + \frac{4 \sigma^2 \log (3 \abs{\cY})}{\eta \beta^2}.
\end{align*}

\subsubsection{Bounding the duality gap}

We use the following lemmas to relate the duality gap with KL divergences, so that we can directly use previous results to establish convergence on the duality gaps.

\begin{lemma} \label{lem:single_side_gap}
For any $\pi$,
\begin{align*}
    V_\beta (\pi_\beta^\star, \pi) - V_\beta (\pi_\beta^\star, \pi_\beta^\star) = \beta \KL (\pi || \pi_\beta^\star), \\
    V_\beta (\pi_\beta^\star, \pi_\beta^\star) - V_\beta (\pi, \pi_\beta^\star) = \beta \KL (\pi || \pi_\beta^\star).
\end{align*}
\end{lemma}

\begin{proof} [Proof of \Cref{lem:single_side_gap}]
We show the proof of the first equation, and the that for the second one is similar.
\begin{align*}
    &V_\beta (\pi_\beta^\star, \pi) - V_\beta (\pi_\beta^\star, \pi_\beta^\star) \\
    &= (\pi_\beta^\star)^\top \cP (\pi - \pi_\beta^\star) - \beta \KL (\pi_\beta^\star || \pi_\tref) + \beta \KL (\pi || \pi_\tref) \\
    &\myeqi - (\pi - \pi_\beta^\star)^\top \cP \pi_\beta^\star - \beta \KL (\pi_\beta^\star || \pi^{(0)}) + \beta \KL (\pi || \pi_\tref) \\
    &\myeqii - \beta (\pi - \pi_\beta^\star)^\top (\theta_\beta^\star - \theta_\tref) - \beta \KL (\pi_\beta^\star || \pi^{(0)}) + \beta \KL (\pi || \pi_\tref) \\
    &\myeqiii - \beta \langle \log \pi_\beta^\star - \log \pi_\tref, \pi - \pi_\beta^\star \rangle - \beta \langle \log \pi_\beta^\star - \log \pi_\tref, \pi_\beta^\star \rangle + \beta \langle \log \pi - \log \pi_\tref, \pi \rangle \\
    &= \beta \langle \log \pi - \log \pi_\beta^\star, \pi \rangle \\
    &= \beta \KL (\pi || \pi_\beta^\star),
\end{align*}
where (i) is by $\cP + \cP^\top = \II$ and $\II (p - q) = \boldsymbol{0}$ where $p, q \in \Delta^{\cY}$;
(ii) is by \Cref{eq:main_equation};
(iii) is by $\langle C \II, p - q \rangle = 0$ where $C \in \bbR$ and $p, q \in \Delta^{\cY}$.
\end{proof}

\begin{lemma} \label{lem:dual_gap}
For any $\pi$,
\begin{align*}
    \dualgap_\beta (\pi) \le \frac{2}{\beta} \KL(\pi_\beta^\star || \pi) + 2 \beta \KL (\pi || \pi_\beta^\star).
\end{align*}
\end{lemma}

\begin{proof} [Proof of \Cref{lem:dual_gap}]
\begin{align*}
    &\dualgap_\beta (\pi) \\
    &= \max_{\pi'} V_\beta (\pi', \pi) - \min_{\pi''} V_\beta (\pi, \pi'') \\
    &= \max_{\pi', \pi''} (V_\beta (\pi', \pi) - V_\beta (\pi, \pi'')) \\
    &= \max_{\pi', \pi''} [\underbrace{(V_\beta (\pi', \pi) - V_\beta (\pi', \pi_\beta^\star))}_{X} - \underbrace{(V_\beta (\pi, \pi'') - V_\beta (\pi_\beta^\star, \pi''))}_{Y} - \underbrace{(V_\beta (\pi_\beta^\star, \pi'') - V_\beta (\pi', \pi_\beta^\star))}_{Z}] \\
    &\myeqi \max_{\pi', \pi''} \bigg[ (\pi' - \pi_\beta^\star)^\top \cP (\pi - \pi_\beta^\star) - (\pi - \pi_\beta^\star)^\top \cP (\pi'' - \pi_\beta^\star) + \underbrace{(V_\beta (\pi_\beta^\star, \pi) - V_\beta (\pi, \pi_\beta^\star))}_W \\
    &\quad\quad\quad\quad - \beta \KL (\pi' || \pi_\beta^\star) - \beta \KL (\pi'' || \pi_\beta^\star)\bigg] \\
    &\myleii \max_{\pi', \pi''} \bigg[(\pi' - \pi_\beta^\star)^\top \cP (\pi - \pi_\beta^\star) - (\pi - \pi_\beta^\star)^\top \cP (\pi'' - \pi_\beta^\star) + 2 \beta \KL (\pi || \pi_\beta^\star) \\
    &\quad\quad\quad\quad - \beta \KL (\pi' || \pi_\beta^\star) - \beta \KL (\pi'' || \pi_\beta^\star)\bigg] \\
    &\myleiii \max_{\pi', \pi''} \bigg(\beta \KL (\pi' || \pi_\beta^\star) + \frac{1}{\beta} \KL(\pi_\beta^\star || \pi) + \frac{1}{\beta} \KL(\pi_\beta^\star || \pi) + \beta \KL(\pi'' || \pi_\beta^\star) + 2 \beta \KL (\pi || \pi_\beta^\star) \\
    &\quad\quad\quad\quad - \beta \KL (\pi' || \pi_\beta^\star) - \beta \KL (\pi'' || \pi_\beta^\star) \bigg) \\
    &= \frac{2}{\beta} \KL(\pi_\beta^\star || \pi) + 2 \beta \KL (\pi || \pi_\beta^\star),
\end{align*}
where (i) is by verifying that $X = (\pi' - \pi_\beta^\star)^\top \cP (\pi - \pi_\beta^\star) + V_\beta (\pi_\beta^\star, \pi) - V_\beta (\pi_\beta^\star, \pi_\beta^\star)$, $Y = (\pi - \pi_\beta^\star)^\top \cP (\pi'' - \pi_\beta^\star) + V_\beta (\pi, \pi_\beta^\star) - V_\beta (\pi_\beta^\star, \pi_\beta^\star)$, and from \Cref{lem:single_side_gap}, $Z = \beta \KL (\pi' || \pi_\beta^\star) + \beta \KL (\pi'' || \pi_\beta^\star)$;
(ii) is by \Cref{lem:single_side_gap}, $W = 2 \beta \KL (\pi || \pi_\beta^\star)$;
(iii) is by \Cref{lem:diff_quad_form} with $\xi = \beta$ and $\xi = 1 / \beta$, respectively.
\end{proof}

\para{For \Cref{eq:extragrad_dualgap}.}
It follows directly from \Cref{lem:dual_gap} and \Cref{eq:extragrad_kl,eq:extragrad_rev_kl}.

\para{For \Cref{eq:extragrad_est_dualgap}.}
It follows directly from \Cref{lem:dual_gap} and \Cref{eq:extragrad_est_rev_kl,eq:extragrad_est_kl}.

\para{For \Cref{thm:extragrad_unreg}.}
Under the condition that $\pi_\tref = \Unif (\cY)$, for any $\pi_1, \pi_2$, we have that
\begin{align*}
    V(\pi_1, \pi_2) - V_\beta (\pi_1, \pi_2)
    &= \beta (\KL (\pi_1 || \pi_\tref) - \KL (\pi_2 || \pi_\tref)) \\
    & = \beta (\KL (\pi_1 || \Unif (\cY)) - \KL (\pi_2 || \Unif (\cY))) \\
    &\le \beta \log \abs{\cY}.
\end{align*}
Then for any $\pi$,
\begin{align*}
    \dualgap (\pi) &= \max_{\pi', \pi''} (V (\pi', \pi) - V (\pi, \pi'')) \\
    &\le \max_{\pi', \pi''} [\abs{V (\pi', \pi) - V_\beta (\pi', \pi)} + \abs{V_\beta (\pi, \pi'') - V (\pi, \pi'')} + (V_\beta (\pi', \pi) - V_\beta (\pi, \pi''))] \\
    &\le 2 \beta \log \abs{\cY} + \max_{\pi', \pi''} (V_\beta (\pi', \pi) - V_\beta (\pi, \pi'')) \\
    &= 2 \beta \log \abs{\cY} + \dualgap_\beta (\pi).
\end{align*}
We set $\beta = \frac{\varepsilon}{4 \log \abs{\cY}}$, so that $2 \beta \log \abs{\cY} = \varepsilon / 2$.
We need to make sure $\dualgap_\beta (\pi) \le \varepsilon / 2$.

Recall that $\pi^{(0)} = \Unif (\cY)$, so $\KL (\pi_\beta^\star || \pi^{(0)}) \le \log \abs{\cY}$.
Let $T = c \cdot \frac{4 \log \abs{\cY}}{\eta \varepsilon}$, in addition that $\sigma^2 = 0$, we have
\begin{align*}
    \dualgap_\beta (\pi^{(T)})
    &\le \left( \frac{8 \log \abs{\cY}}{\varepsilon} + \frac{4}{\eta} \right) \log \abs{\cY} \left[\left(1 - \frac{\eta \varepsilon}{4 \log \abs{\cY}}\right)^{\frac{4 \log \abs{\cY}}{\eta \varepsilon}} \right]^c \\
    &\le \left( \frac{8 \log \abs{\cY}}{\varepsilon} + \frac{4}{\eta} \right) \log \abs{\cY} \ee^{-c}.
\end{align*}
So setting $c = \log \left(\frac{2}{\varepsilon} \left( \frac{8 \log \abs{\cY}}{\varepsilon} + \frac{4}{\eta} \right) \log \abs{\cY} \right)$ makes $\dualgap_\beta (\pi^{(T)}) \le \varepsilon / 2$.
This implies $T = \wt{\Theta} (1 / \varepsilon)$ if we choose $\eta = \frac{1}{\beta + 3} = \Theta (1)$.
It is similar for $\dualgap (\pi^{(T+1/2)})$.

\subsection{Discussions on empirical updates for baselines} \label{sec:baseline_empirical}

\subsubsection{Nash-MD and INPO}

These two algorithms use similar proof techniques, so we use Nash-MD as an example.

Combining Equation (4) and Lemmas 1 and 2 in \citet{munos2023nash}, the closed-form update in \Cref{sec:ipo_nlhf}, and the proof in \Cref{sec:proof_extragrad}, we obtain the following result when $\eta \le 1 / \beta$:
\begin{align*}
    \bbE [\KL (\pi_\beta^\star || \pi^{(t+1)})] \le (1 - \eta \beta) \bbE [\KL (\pi_\beta^\star || \pi^{(t)})] + c_1 \eta^2 + c_2 \sigma^2 \log (3\abs{\cY}),
\end{align*}
where $c_1$ and $c_2$ are absolute constants.
This transforms into:
\begin{align*}
    \bbE [\KL (\pi_\beta^\star || \pi^{(T)})] \le (1 - \eta \beta)^T \KL (\pi_\beta^\star || \pi^{(0)}) + \frac{c_1 \eta^2 + c_2 \sigma^2 \log (3\abs{\cY})}{\eta \beta}.
\end{align*}
We can see that the constant term is approximately $\eta + \sigma^2 / \eta$, so it is lower-bounded by $\sigma$ and the choice of $\eta$ could be constrained.

If we follow the original choice of $\eta = \log T / (\beta T)$, then:
\begin{align*}
    \bbE [\KL (\pi_\beta^\star || \pi^{(T)})] \le\left( \KL (\pi_\beta^\star || \pi^{(0)}) + \frac{c_1 \log T}{\beta^2} \right) \frac{1}{T} + \frac{c_2 \sigma^2 \log (3 \abs{\cY})}{\log T} T,
\end{align*}
which is nonsensical when $\sigma > 0$.

When $0 < \sigma \le 1 / \beta$, choosing $\eta = \sigma$ yields:
\begin{align*}
    \bbE [\KL (\pi_\beta^\star || \pi^{(T)})] \le (1 - \sigma \beta)^T \KL (\pi_\beta^\star || \pi^{(0)}) + \frac{(c_1 + c_2) \sigma \log (3\abs{\cY})}{\beta},
\end{align*}
which could be substantially slower compared to \Cref{eq:extragrad_kl} when $\sigma$ approaches $0$.

When $\sigma > 1 / \beta$, choosing $\eta = 1 / \beta$ gives:
\begin{align*}
    \bbE [\KL (\pi_\beta^\star || \pi^{(T)})] \le \frac{c_1}{\beta^2} + c_2 \sigma^2 \log (3\abs{\cY}),
\end{align*}
which could be slower than \Cref{eq:extragrad_kl} when $\beta$ is small.



\subsection{Discussions on convergence to the original NE for baselines} \label{sec:baseline_original_ne}

\subsubsection{Nash-MD}

Only $\KL (\pi_\beta^\star || \pi^{(T)})$ is bounded in \citet{munos2023nash}, and we are not clear about the bound of $\KL (\pi^{(T)} || \pi_\beta^\star)$.
This makes it hard to directly apply \Cref{lem:dual_gap} in our work.
Thus, we cannot make arguments on convergence of either $\dualgap_\beta$ or $\dualgap$ for Nash-MD, hence its convergence to the original NE is unclear.

\subsubsection{INPO}

We take $\pi_\tref = \Unif (\cY)$.
Theorem~3 in \citet{zhang2024iterative} states that
\begin{align*}
    \dualgap_\beta \left( \frac{1}{T} \sum_{t=1}^T \pi^{(t)} \right) \le \frac{\max\{B\beta, 1\}\sqrt{\log{\abs{\cY}}}}{\sqrt{T}},
\end{align*}
where $B$ (from their Assumption~A) is the upper bound for any time log ratio:
\begin{align*}
    B = \sup_{\textup{Any training process~} \pi^{(0)}, \ldots, \pi^{(T)}} \max_t \norm{\log \frac{\pi^{(t)}}{\pi_\tref}}_{\infty}.
\end{align*}
The authors did not give the value of $B$, as opposed to the bound of $\sigma_{\min}'$ in Theorems~1, 4 and 6 of \citet{shi2025the}.
In fact, bounding $B$ is closely related to the algorithm design and not straightforward.
\red{\textbf{Assume}} the maximum value of $B$ is taken when $\pi^{(T)} = \pi_\beta^\star$, then $B \le 1 / \beta$.
So, $\dualgap_\beta \le \wt{O} (1 / \sqrt{T})$.
Using the same argument in the proof for \Cref{thm:extragrad_unreg}, we can show a $\wt{O} (1 / \varepsilon^2)$ iteration complexity.
Note that this result is only average-iterate convergence.

\subsubsection{MPO}

Theorem F.1 in \citet{wang2024magneticpreferenceoptimizationachieving} states that MPO satisfies $\dualgap_\beta (\pi^{(T)}) \le \wt{O} ((\frac{1}{1 + \eta \beta})^{T/2})$.
Using a similar argument as in our proof for \Cref{thm:extragrad_unreg}, by setting $\beta = \frac{\varepsilon}{4 \log \abs{\cY}}$, we have that for any $T \ge \wt{\Omega} (\frac{\log (1 / \varepsilon)}{\log (1 + \eta \beta)}) = \wt{\Omega} (1 / (\eta \beta))$, $\dualgap (\pi^{(T)}) \le \varepsilon$.
However, Theorem 3.2 in \citet{wang2024magneticpreferenceoptimizationachieving} states that we can only choose $\eta \le \beta$.
So, the iteration complexity is $\wt{O} (1 / \varepsilon^2)$.

\subsection{Online IPO} 

\subsubsection{Justification for equivalence between \eg and online IPO} \label{sec:proof_eg_oipo}

Recall the generalized IPO loss:
\begin{align*}
    \cL_\ipo (\theta; \rho, \mu)
    &= \bbE_{(y, y') \sim \rho} \left[ \left( \log \frac{\pi_\theta(y) \pi_\tref (y')}{\pi_\theta (y') \pi_\tref (y)} - \frac{1}{\beta} \bbE_{y'' \sim \mu} [\cP (y \succ y'') - \cP (y' \succ y'')] \right)^2 \right] \\
    &= \bbE_{(y, y') \sim \rho} \left[ \left( \left(\theta - \theta_\tref - \frac{\cP \mu}{\beta}\right)^\top (\II_y - \II_{y'}) \right)^2 \right].
\end{align*}
Define $\Sigma (\rho) := \bbE_{(y, y') \sim \rho} [ (\II_y - \II_{y'}) (\II_y - \II_{y'})^\top ]$, then
\begin{align*}
    \nabla_\theta \cL_\ipo (\theta; \rho, \mu) = 2 \bbE_{(y, y') \sim \rho} \left[ \left(\theta - \theta_\tref - \frac{\cP \mu}{\beta}\right)^\top (\II_y - \II_{y'}) \cdot (\II_y - \II_{y'}) \right] 
    = 2 \Sigma (\rho) \left(\theta - \theta_\tref - \frac{\cP \mu}{\beta}\right).
\end{align*}

The QRE satisfies $\forall y, y' \in \cY$,
\begin{align*}
    \log \frac{\pi_\beta^\star(y) \pi_\tref (y')}{\pi_\beta^\star (y') \pi_\tref (y)} = \frac{1}{\beta} \bbE_{y'' \sim \pi_\beta^\star} [\cP (y \succ y'') - \cP (y' \succ y'')].
\end{align*}
This transforms to an online IPO loss function:
\begin{align*}
    \cL_\ipo (\theta; \samp, \red{\sg{\pi_\theta}})
    &= \bbE_{(y, y') \sim \samp} \left[ \left( \left(\theta - \theta_\tref - \frac{\cP \red{\sg{\pi_\theta}}}{\beta}\right)^\top (\II_y - \II_{y'}) \right)^2 \right], \\
    \nabla_\theta \cL_\ipo (\theta; \samp, \red{\sg{\pi_\theta}}) &= 2 \Sigma(\samp) \left(\theta - \theta_\tref - \frac{\cP \pi_\theta}{\beta}\right).
\end{align*}
Clearly, $\theta_\beta^\star$ is the minimizer of this loss function as $\cL_\ipo (\theta_\beta^\star; \samp, \pi_\beta^\star) = 0$.

From $\Sigma (\samp) = \frac{2}{\abs{\cY}^2} (\abs{\cY} I - \II)$, we have
\begin{align*}
    \nabla_\theta \cL_\ipo (\theta; \samp, \mu) = \frac{4}{\abs{\cY}} \left(\theta - \theta_\tref - \frac{\cP \mu}{\beta}\right) + C \II.
\end{align*}
Comparing with the coefficients of \Cref{eq:extragrad_est,eq:extragrad_upd}, we know that the update defined by \Cref{eq:ipo_est,eq:ipo_upd} is equivalent to \eg.

\subsubsection{Proof of the population loss} \label{sec:proof_oipo}

\begin{proof} [Proof of \Cref{thm:pop_ipo_loss}]
\begin{align*}
    \cL_\ipo (\theta; \samp, \mu)
    &= \bbE_{(y, y') \sim \samp} \left[ \left( \log \frac{\pi_\theta(y) \pi_\tref (y')}{\pi_\theta (y') \pi_\tref (y)} - \frac{\cP (y \succ \mu) - \cP (y' \succ \mu)}{\beta} \right)^2 \right] \\
    &= \bbE_{(y, y') \sim \samp} \left[ \left( \log \frac{\pi_\theta(y) \pi_\tref (y')}{\pi_\theta (y') \pi_\tref (y)} \right)^2 \right] \\
    &\quad - \frac{2}{\beta} \bbE_{(y, y') \sim \samp} \left[ \log \frac{\pi_\theta(y) \pi_\tref (y')}{\pi_\theta (y') \pi_\tref (y)} \cdot (\cP (y \succ \mu) - \cP (y' \succ \mu)) \right] \\
    &\quad + \frac{1}{\beta^2} \bbE_{(y, y') \sim \samp} \left[ (\cP (y \succ \mu) - \cP (y' \succ \mu))^2 \right].
\end{align*}
The first term is easy to estimate unbiasedly.
The last term does not contribute to $\nabla_\theta \cL_\ipo (\theta; \samp, \mu)$.
We focus on the second term.
\begin{align*}
    & \bbE_{(y, y') \sim \samp} \left[ \log \frac{\pi_\theta(y) \pi_\tref (y')}{\pi_\theta (y') \pi_\tref (y)} \cdot (\cP (y \succ \mu) - \cP (y' \succ \mu)) \right] \\
    &= \bbE_{(y, y') \sim \samp} \left[ \log \frac{\pi_\theta(y)}{\pi_\tref (y)} \cdot \cP (y \succ \mu) \right] - \bbE_{(y, y') \sim \samp} \left[ \log \frac{\pi_\theta(y)}{\pi_\tref (y)} \cdot \cP (y' \succ \mu) \right] \\
    &\quad - \bbE_{(y, y') \sim \samp} \left[ \log \frac{\pi_\theta(y')}{\pi_\tref (y')} \cdot \cP (y \succ \mu) \right] + \bbE_{(y, y') \sim \samp} \left[ \log \frac{\pi_\theta(y')}{\pi_\tref (y')} \cdot \cP (y' \succ \mu) \right] \\
    &= 2 \bbE_{y \sim \samp} \left[ \log \frac{\pi_\theta(y)}{\pi_\tref (y)} \cdot \cP (y \succ \mu) \right] - 2 \cP(\samp \succ \mu) \bbE_{y \sim \samp} \left[ \log \frac{\pi_\theta(y)}{\pi_\tref (y)}\right].
\end{align*}

Recall \Cref{eq:pop_ipo_loss}:
\begin{align*}
    \bbE_{(y, y') \sim \samp, \red{y'' \sim \mu}} \left[ \left( \log \frac{\pi_\theta(y) \pi_\tref (y')}{\pi_\theta (y') \pi_\tref (y)} - \frac{I (y, \red{y''}) - I (y', \red{y''})}{\beta} \right)^2 \right],
\end{align*}
Clearly, we only need to examine the cross term:
\begin{align*}
    & \bbE_{(y, y') \sim \samp, y'' \sim \mu} \left[ \log \frac{\pi_\theta(y) \pi_\tref (y')}{\pi_\theta (y') \pi_\tref (y)} \cdot (I (y, y'') - I (y', y''))\right] \\
    &= \bbE_{(y, y') \sim \samp, y'' \sim \mu} \left[ \log \frac{\pi_\theta(y) \pi_\tref (y')}{\pi_\theta (y') \pi_\tref (y)} \cdot I (y, y'') \right] - \bbE_{(y, y') \sim \samp, y'' \sim \mu} \left[ \log \frac{\pi_\theta(y) \pi_\tref (y')}{\pi_\theta (y') \pi_\tref (y)} \cdot I (y', y'') \right] \\
    &= 2 \bbE_{(y, y') \sim \samp, y'' \sim \mu} \left[ \log \frac{\pi_\theta(y) \pi_\tref (y')}{\pi_\theta (y') \pi_\tref (y)} \cdot I (y, y'') \right] \\
    &= 2 \bbE_{(y, y') \sim \samp} \left[ \log \frac{\pi_\theta(y)}{\pi_\tref (y)} \cdot \bbE_{y'' \sim \mu} [I (y, y'') | y]\right] - 2 \bbE_{(y, y') \sim \samp} \left[ \log \frac{\pi_\theta(y')}{\pi_\tref (y')} \cdot \bbE_{y'' \sim \mu} [I (y, y'') | y]\right] \\
    &= 2 \bbE_{(y, y') \sim \samp} \left[ \log \frac{\pi_\theta(y)}{\pi_\tref (y)} \cdot \cP (y \succ \mu) \right] - 2 \bbE_{(y, y') \sim \samp} \left[ \log \frac{\pi_\theta(y')}{\pi_\tref (y')} \cdot \cP (y \succ \mu) \right] \\
    &= 2 \bbE_{y \sim \samp} \left[ \log \frac{\pi_\theta(y)}{\pi_\tref (y)} \cdot \cP (y \succ \mu) \right] - 2 \cP(\samp \succ \mu) \bbE_{y \sim \samp} \left[ \log \frac{\pi_\theta(y)}{\pi_\tref (y)}\right].
\end{align*}
By comparing coefficients, we have that the gradients of $\cL_\ipo (\theta; \samp, \mu)$ and \Cref{eq:pop_ipo_loss} are equivalent.
\end{proof}
\section{Experiment details} \label{sec:exp_details}

Here we present more experiment details (including settings and results) that are omitted in the main text.

\subsection{Implementation of baselines} \label{sec:ipo_nlhf}

\para{Online IPO 1 (OMD).}
OMD is shown as Equation (8) in \citet{munos2023nash}:
\begin{align*}
    \pi^{(t+1)} = \arg\max_\pi \{ \eta \cP (\pi \succ \pi^{(t)}) - \KL (\pi || \wt{\pi}^{(t)}) \},
\end{align*}
where $\wt{\pi}^{(t)} (y) \propto (\pi^{(t)} (y))^{1 - \eta \beta} (\pi_\tref (y))^{\eta \beta}$.
This is equivalent to
\begin{align*}
    \theta^{(t+1)} = \theta^{(t)} - \eta \beta \left(\theta^{(t)} - \theta_\tref - \frac{\cP \pi^{(t)}}{\beta} \right).
\end{align*}
So the update is simply the $\pi^{(t+1/2)}$ part of Extragradient:
\begin{align*}
    \theta^{(t+1)} &\gets \theta^{(t)} - \frac{\eta \beta \abs{\cY}}{4} \nabla_\theta \cL_\ipo (\theta^{(t)}; \samp, \red{\sg{\pi^{(t)}}}).
\end{align*}
Thus OMD is a type of online IPO with uniform sampling for response pairs and online sampling for preference comparison.

\para{Online IPO 2.}
Online IPO 2 \citep{ye2024online,calandriello2024humanalignmentlargelanguage} uses the loss function of $\wh{\cL} (\theta; \red{\sg{\pi_\theta}}, \red{\sg{\pi_\theta}})$ (see \Cref{thm:pop_ipo_loss}), which is equivalent to the following update:
\begin{align*}
    \theta^{(t+1)} &\gets \theta^{(t)} - \frac{\eta \beta \abs{\cY}}{4} \nabla_\theta \cL_\ipo (\theta^{(t)}; \red{\sg{\pi^{(t)}}}, \red{\sg{\pi^{(t)}}}).
\end{align*}
Its population loss has a variance-reduced formulation \citep{Azar2023AGT,ye2024online,calandriello2024humanalignmentlargelanguage} where no $y''$ is needed:
\begin{align*}
    \wh{\cL} (\theta) := \bbE_{(y, y') \sim \red{\sg{\pi_\theta}}} \left[ \left( \log \frac{\pi_\theta(y) \pi_\tref (y')}{\pi_\theta (y') \pi_\tref (y)} - \frac{I (y, y') - \frac{1}{2}}{\beta} \right)^2 \right].
\end{align*}

\para{Nash-MD.}
Nash-MD is shown as Equation (4) in \citet{munos2023nash}:
\begin{align}
    \pi^{(t+1)} = \arg\max_\pi \{ \eta \cP (\pi \succ \wt{\pi}^{(t)}) - \KL (\pi || \wt{\pi}^{(t)}) \}, \label{eq:nash_md}
\end{align}
which is equivalent to
\begin{align*}
    \theta^{(t+1)} = \theta^{(t)} - \eta \beta \left(\theta^{(t)} - \theta_\tref - \frac{\cP \wt{\pi}^{(t)}}{\beta} \right).
\end{align*}
So the update is
\begin{align*}
    \theta^{(t+1)} &\gets \theta^{(t)} - \frac{\eta \beta \abs{\cY}}{4} \nabla_\theta \cL_\ipo (\theta^{(t)}; \samp, \red{\sg{\wt{\pi}^{(t)}}}).
\end{align*}

\para{Nash-MD-PG.} 
Instead of directly solving the $\arg\max$ problem, Nash-MD-PG (see Section~7.3 in \citet{munos2023nash}) does a policy gradient update on the inner objective of \Cref{eq:nash_md}:
\begin{align*}
    \theta^{(t+1)} = \theta^{(t)} + \eta \bbE_{y \sim \pi^{(t)}} \left[ \left(P \wt{\pi}^{(t)} - \beta \log \frac{\pi_\theta(y)}{\pi_\tref (y)} \right) \nabla_\theta \log \pi^{(t)} (y) \right].
\end{align*}
This update can be approximated using two samples per term: $y \sim \pi^{(t)}$ and $y' \sim \wt{\pi}^{(t)}$.

\subsection{Numerical simulations} \label{sec:more_numerical}

\subsubsection{Experiment setups}

\para{Preference matrix.}
We first fill the lower triangle of $\cP$ with each element i.i.d. from $\Unif([0, 1])$, then set the diagonal elements to be $\frac{1}{2}$ and complement the upper triangle with corresponding values.
For tabular experiments, we set $\abs{\cY} = 10$; for neural network experiments, we set $\abs{\cY} = 100$.

\para{Neural network architecture.}
We use a $3$-layer MLP with ReLU activation as the neural policy.
The hidden dimension $d$ is set to be $10$.
Since we consider multi-armed bandit environments, there is no input to this policy.
Hence, we use a random Gaussian noise $\cN (0, I_d)$ as input.

\para{Reference policy.}
For tabular policies, we sample the parameters from $\cN (0, I_{\abs{\cY}})$ as reference policies.
For neural policies, we use Xavier normal initialization \citep{glorot2010understanding}.

\subsubsection{Convergence of duality gaps}

\Cref{fig:sim_tab_dual_gap_beta_0.001,fig:sim_tab_dual_gap_beta_0.01,fig:sim_tab_dual_gap_beta_0.1} are results of the algorithms under the \textbf{exact gradient} setting using \textbf{tabular} policy class, with different $\beta$s and $\eta$s (values specified in the captions).
Same as in the main text, values are cut off below $10^{-6}$ due to floating point precision.

We also report results of the other experiments in \Cref{fig:sim_emp_tab_dual_gap,fig:sim_neural_dual_gap_beta_0.01,fig:sim_emp_neural_dual_gap_beta_0.1}.
For \textbf{empirical} algorithms, we use $100$ samples per update to estimate the loss function $\wh{\cL}$ in \Cref{thm:pop_ipo_loss}.

\begin{figure}[htbp]
    \centering
    \includegraphics[width=\dimexpr0.98\linewidth\relax, height=\dimexpr0.98\textheight\relax, keepaspectratio]{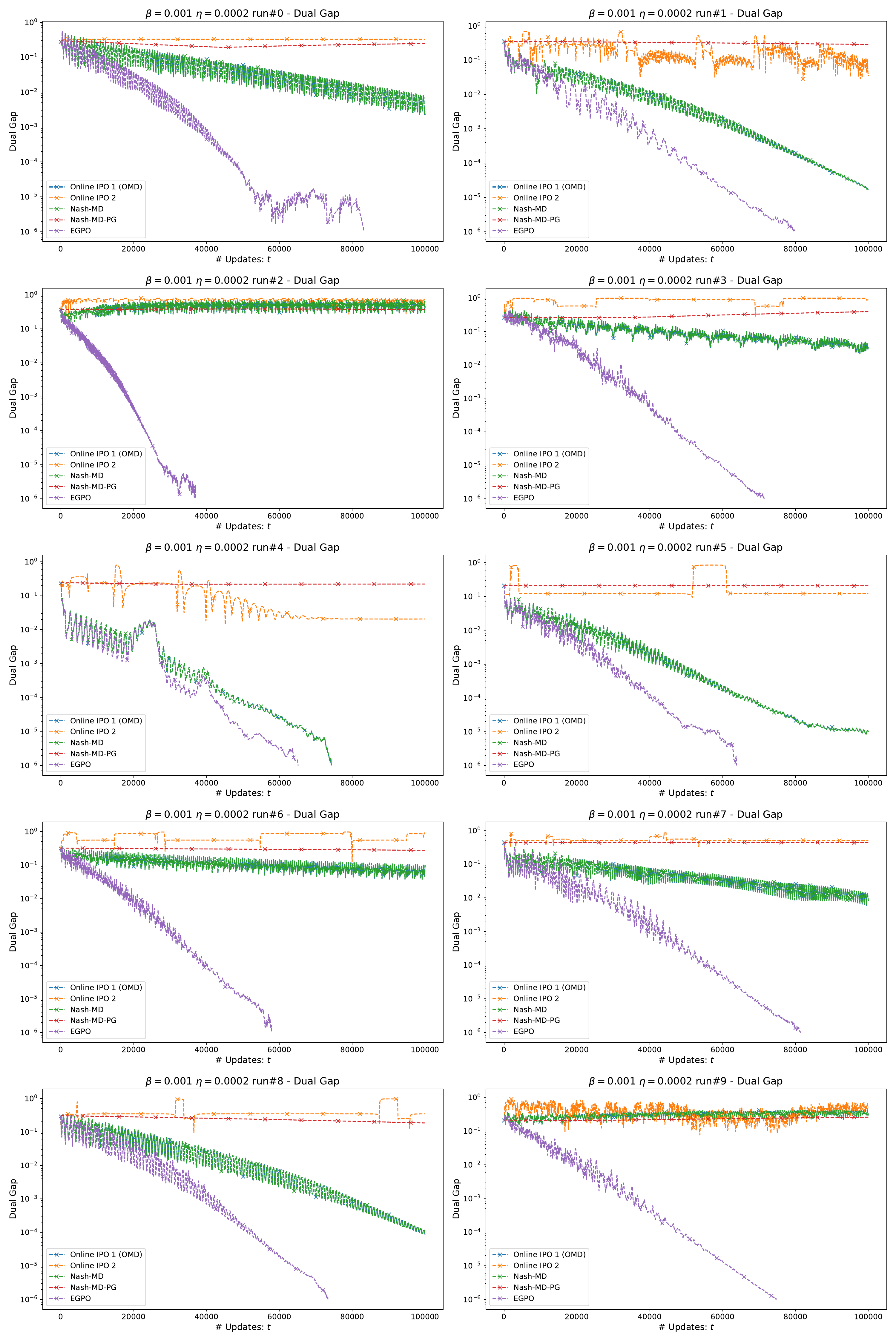}
    \caption{Duality gap ($\dualgap_\beta$) of \textbf{exact tabular} algorithms with $\beta = 0.001$ and $\eta = 0.0002$.}
    \label{fig:sim_tab_dual_gap_beta_0.001}
\end{figure}

\begin{figure}[htbp]
    \centering
    \includegraphics[width=\dimexpr0.98\linewidth\relax, height=\dimexpr0.98\textheight\relax, keepaspectratio]{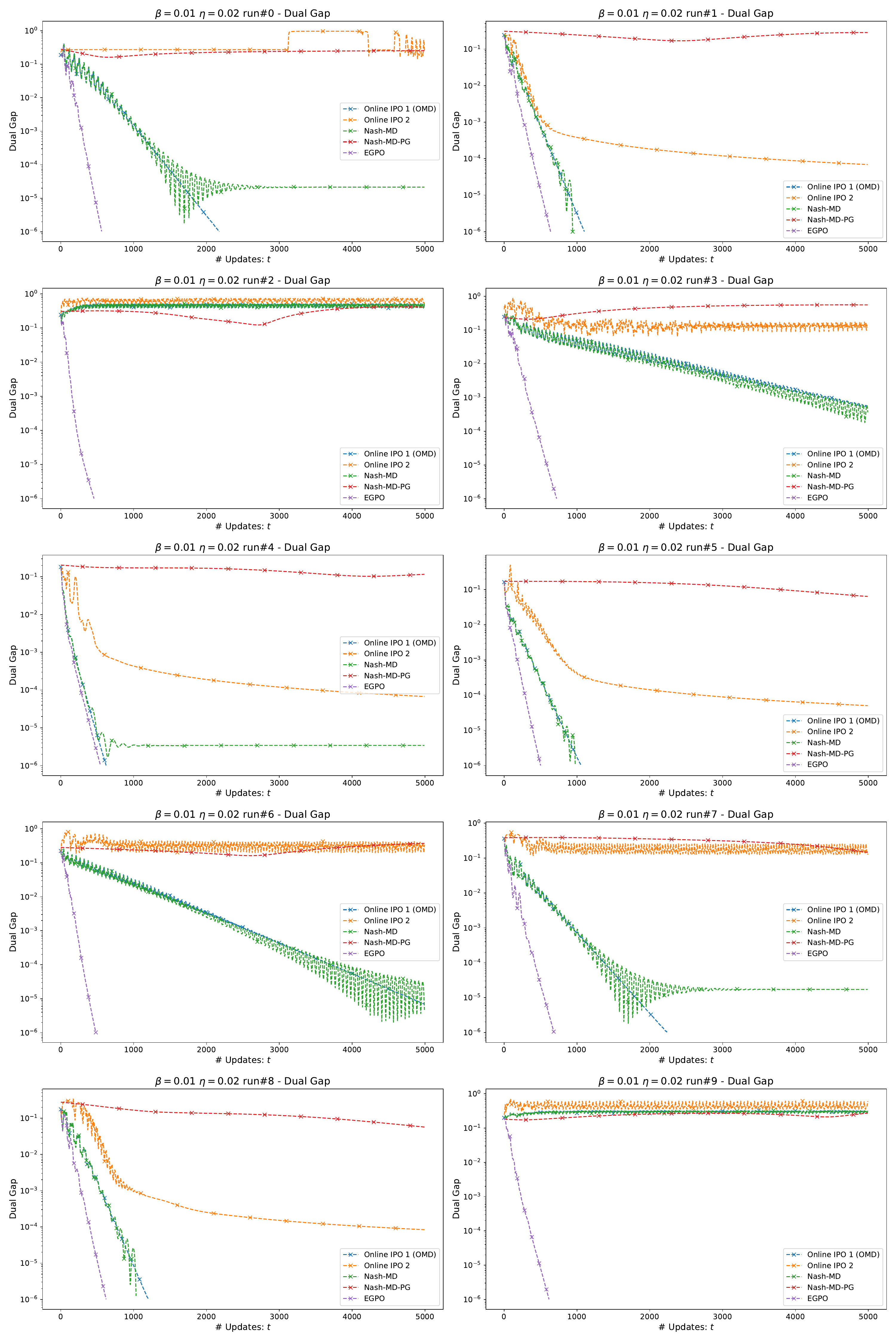}
    \caption{Duality gap ($\dualgap_\beta$) of \textbf{exact tabular} algorithms with $\beta = 0.01$ and $\eta = 0.02$.}
    \label{fig:sim_tab_dual_gap_beta_0.01}
\end{figure}

\begin{figure}[htbp]
    \centering
    \includegraphics[width=\dimexpr0.98\linewidth\relax, height=\dimexpr0.98\textheight\relax, keepaspectratio]{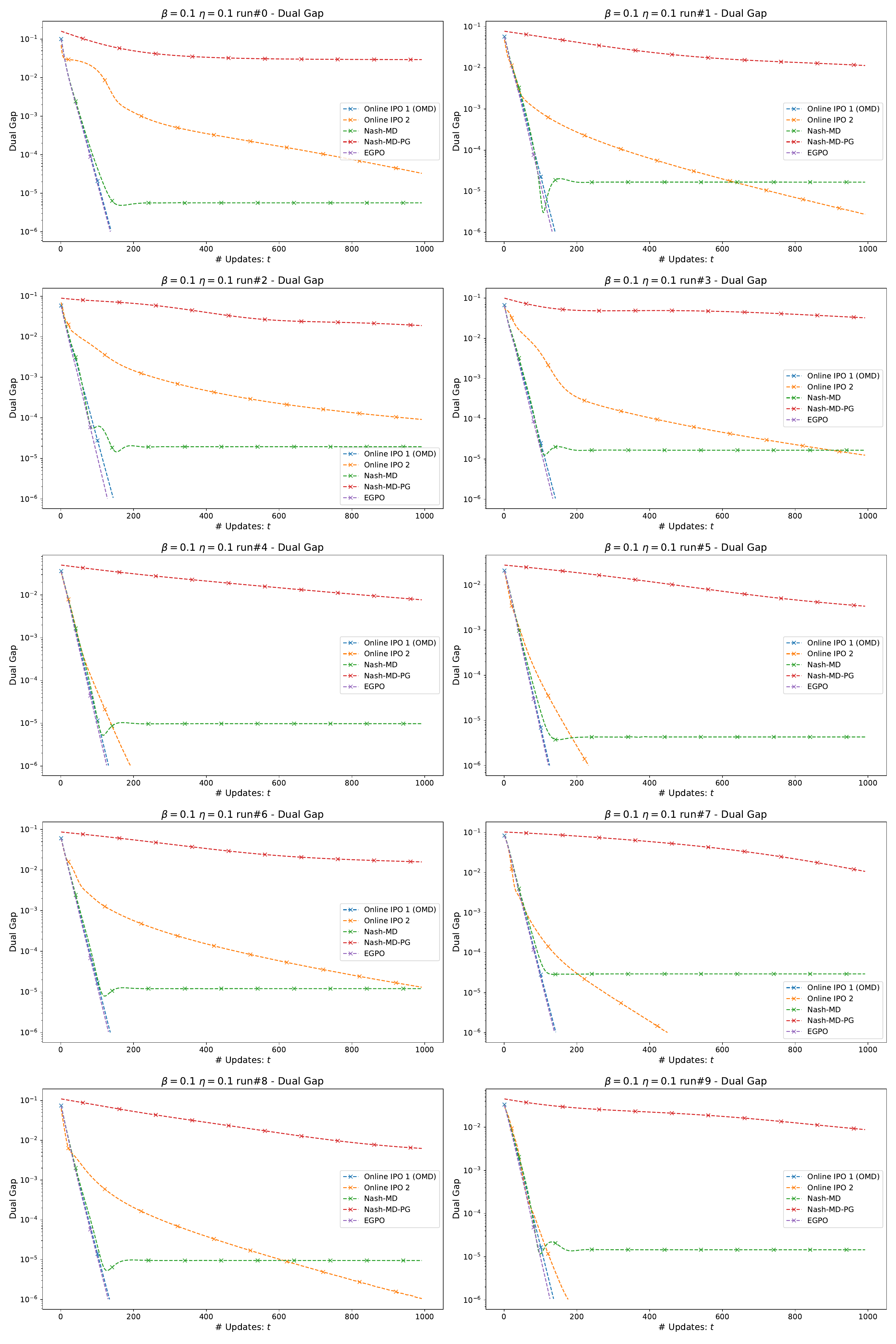}
    \caption{Duality gap ($\dualgap_\beta$) of \textbf{exact tabular} algorithms with $\beta = 0.1$ and $\eta = 0.1$.}
    \label{fig:sim_tab_dual_gap_beta_0.1}
\end{figure}

\begin{figure}[htbp]
    \centering
    \includegraphics[width=\dimexpr0.98\linewidth\relax, height=\dimexpr0.98\textheight\relax, keepaspectratio]{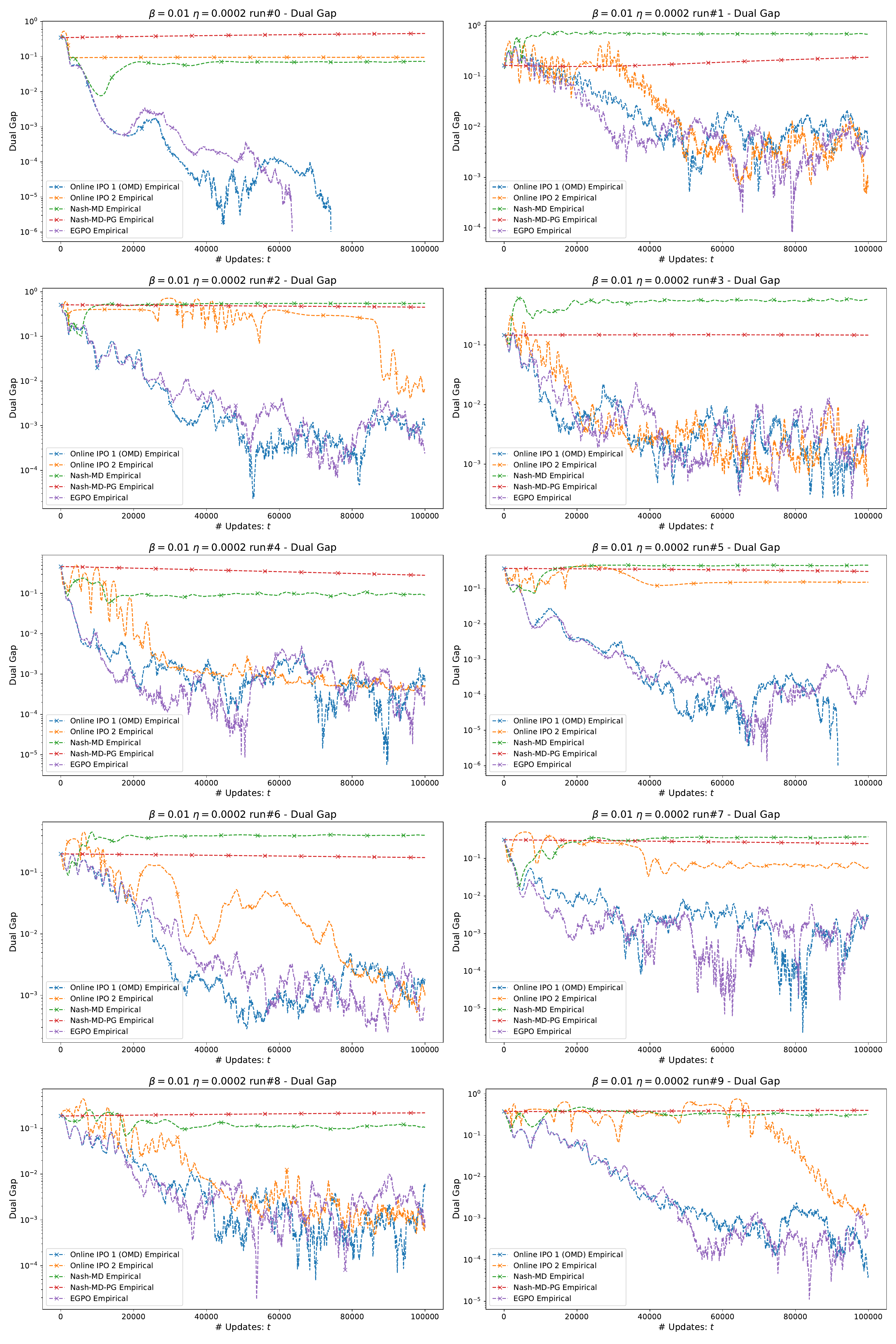}
    \caption{Duality gap ($\dualgap_\beta$) of \textbf{empirical tabular} algorithms with $\beta = 0.01$ and $\eta = 0.0002$.}
    \label{fig:sim_emp_tab_dual_gap}
\end{figure}

\begin{figure}[htbp]
    \centering
    \includegraphics[width=\dimexpr0.98\linewidth\relax, height=\dimexpr0.98\textheight\relax, keepaspectratio]{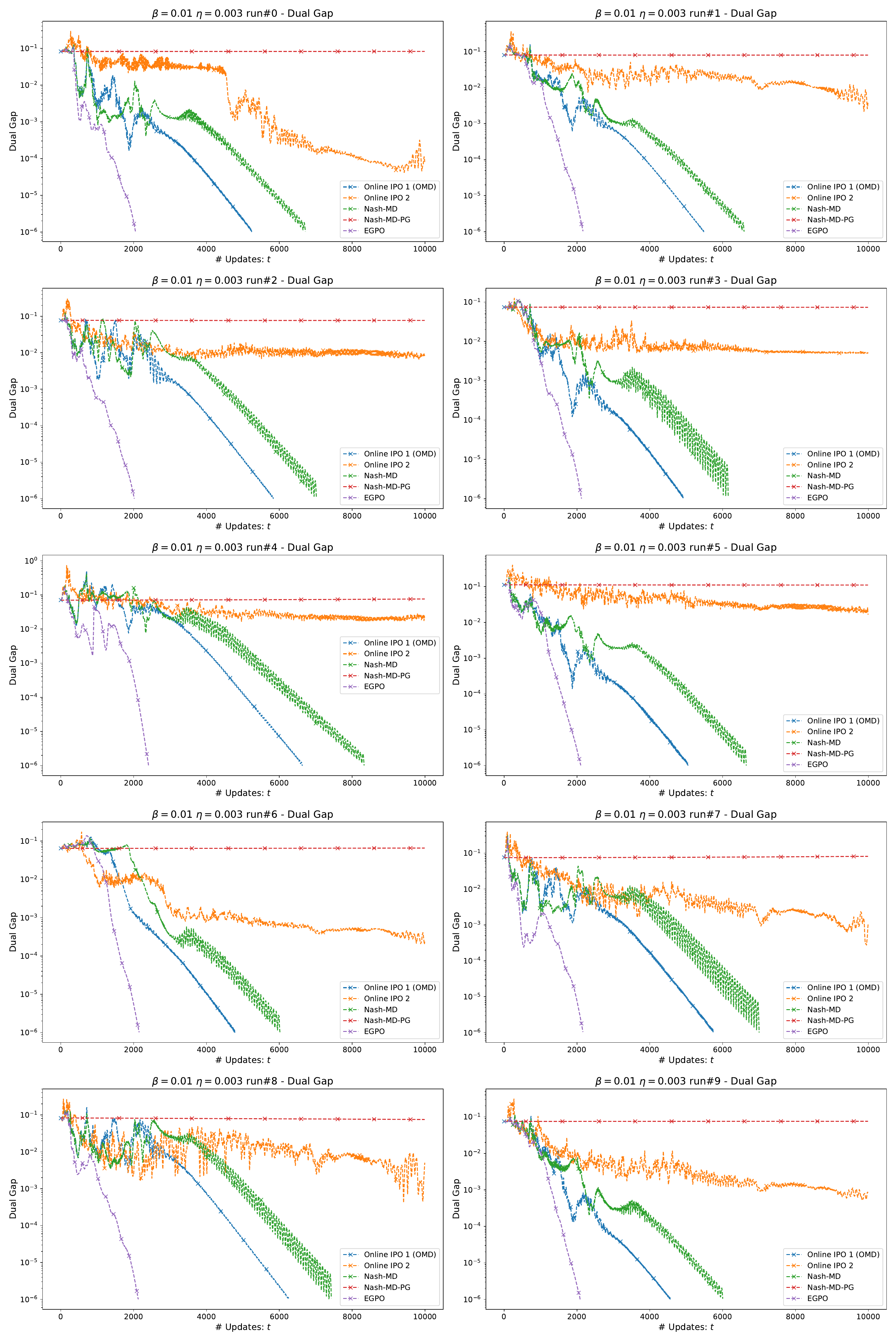}
    \caption{Duality gap ($\dualgap_\beta$) of \textbf{exact neural} algorithms with $\beta = 0.01$ and $\eta = 0.003$.}
    \label{fig:sim_neural_dual_gap_beta_0.01}
\end{figure}

\begin{figure}[htbp]
    \centering
    \includegraphics[width=\dimexpr0.98\linewidth\relax, height=\dimexpr0.98\textheight\relax, keepaspectratio]{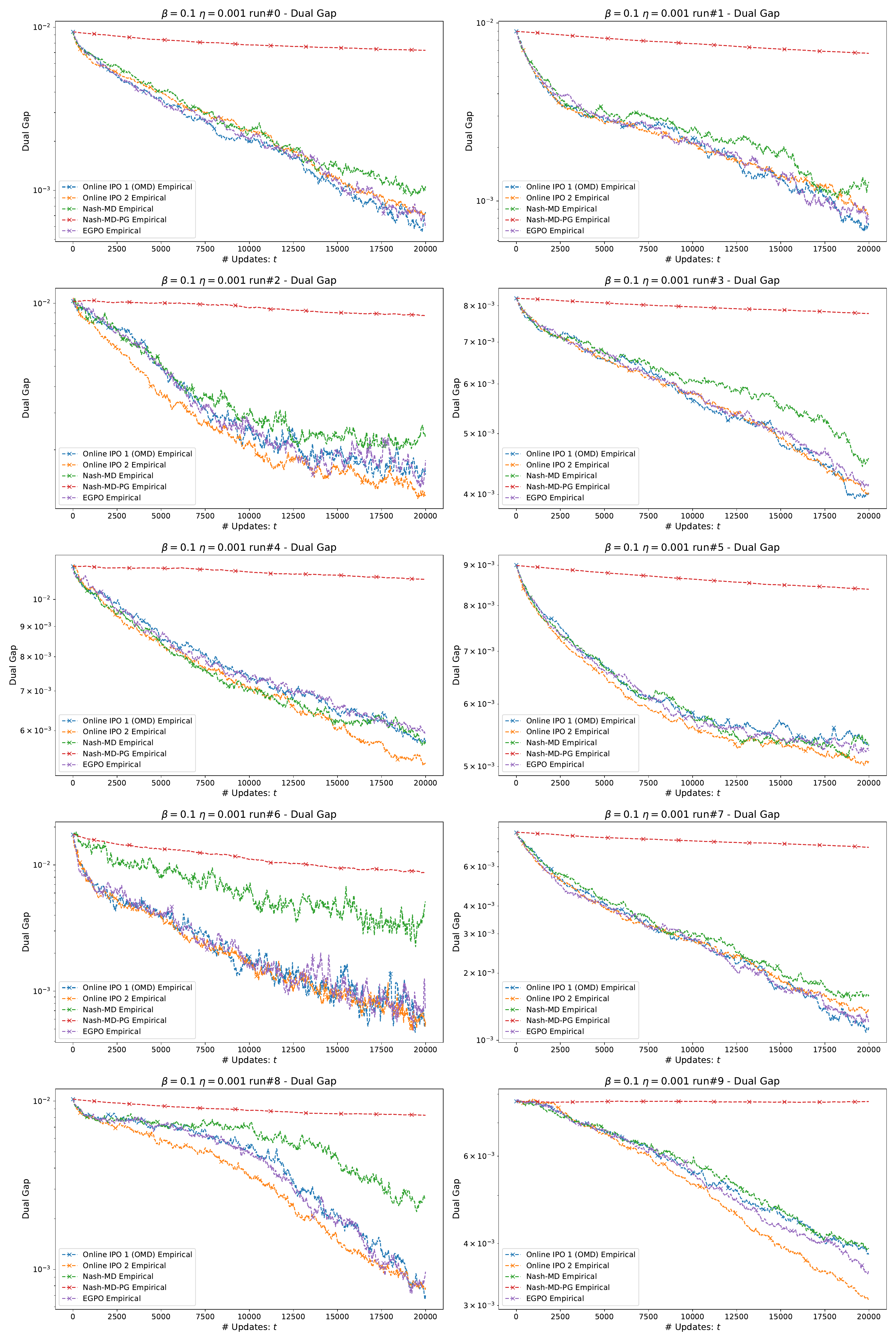}
    \caption{Duality gap ($\dualgap_\beta$) of \textbf{empirical neural} algorithms with $\beta = 0.1$ and $\eta = 0.001$.}
    \label{fig:sim_emp_neural_dual_gap_beta_0.1}
\end{figure}

\subsection{Language model alignments} \label{sec:more_lm}

\subsubsection{Experiment setups} \label{sec:original_exp}

\para{Ground truth preference.}
As we stated previously, there is no \emph{preference modeling} in our NLHF pipeline.
Due to resource constraints, we use a local small language model as a surrogate for human annotators.
\textbf{Queries to this model can be easily delegated to API calls to other LLMs or humans.}
We SFT a \texttt{gemma-2-2b-it} model for sequence classification on a mixture of widely-used open-source preference datasets\footnote{\url{https://huggingface.co/datasets/weqweasdas/preference\_dataset\_mixture2\_and\_safe\_pku}} as the ground truth preference $\cP$.
The input template for this model is shown in \Cref{text:pref_template}.
This model is full-finetuned with all trainable parameters, with detailed settings listed in \Cref{tab:pref_hparam}.

\begin{textbox}[htbp]
\begin{textframe}
\begin{lstlisting}[language=text]
I require a leaderboard for various large language models. I'll provide you with prompts given to these models and their corresponding outputs. Your task is to assess these responses, and select the model that produces the best output from a human perspective.

## Instruction

{{
    "instruction": """{prompt}""",
}}

## Model Outputs

Here are the unordered outputs from the models. Each output is associated with a specific model, identified by a unique model identifier.

{{
    {{
        "model_identifier": "0",
        "output": """{response0}"""
    }},
    {{
        "model_identifier": "1",
        "output": """{response1}"""
    }}
}}
\end{lstlisting}
\end{textframe}
\caption{The input template for the ground truth preference.}
\label{text:pref_template}
\end{textbox}

\begin{table*}[htbp]
    \centering
    
    \begin{tabular}{l | l}
        \hline
        \textbf{Hyperparameter} & \textbf{Value}  \\
        \hhline{=|=}
        Number of epochs & $3$ \\
        \hline
        Train batch size & $64$ \\
        \hline
        Optimizer & AdamW \\
        - Gradient clipping norm & $1.0$ \\
        - $\beta_1, \beta_2$ & $0.9, 0.999$ \\
        - $\epsilon$ & $1 \times 10^{-6}$ \\
        - Weight decay & $0.1$ \\
        \hline
        Learning rate scheduler & WarmupLR \\
        - Warmup max lr & $1 \times 10^{-5}$ \\
        - Warmup steps & $1000$ \\
        - Warmup type & Linear \\
        \hline
        Precision & bf$16$\\
        \hline
        Sequence length & $1024$ \\
        \hline
    \end{tabular}
    \caption{Hyperparameters of SFT for the ground truth preference $\cP$.}
    \label{tab:pref_hparam}
\end{table*}

\para{Reference policy.}
We SFT another \texttt{gemma-2-2b-it} model for causal language modeling on the Alpaca dataset\footnote{\url{https://huggingface.co/datasets/yahma/alpaca-cleaned}} as the reference policy $\pi_\tref$ and the initialization $\pi^{(0)}$.
This model is full-finetuned with all trainable parameters, with detailed settings listed in \Cref{tab:pi_ref_hparam}.

\begin{table*}[htbp]
    \centering
    
    \begin{tabular}{l | l}
        \hline
        \textbf{Hyperparameter} & \textbf{Value}  \\
        \hhline{=|=}
        Number of epochs & $5$ \\
        \hline
        Train batch size & $256$ \\
        \hline
        Optimizer & AdamW \\
        - Gradient clipping norm & $1.0$ \\
        - $\beta_1, \beta_2$ & $0.9, 0.999$ \\
        - $\epsilon$ & $1 \times 10^{-6}$ \\
        - Weight decay & $0.1$ \\
        \hline
        Learning rate scheduler & WarmupDecayLR \\
        - Warmup max lr & $1 \times 10^{-5}$ \\
        - Warmup steps & $100$ \\
        - Warmup type & Linear \\
        \hline
        Precision & bf$16$\\
        \hline
        Sequence length & $512$ \\
        \hline
    \end{tabular}
    \caption{Hyperparameters of SFT for the reference policy $\pi_\tref$.}
    \label{tab:pi_ref_hparam}
\end{table*}

\para{NLHF training.}
We choose the PKU-SafeRLHF dataset\footnote{\url{https://huggingface.co/datasets/PKU-Alignment/PKU-SafeRLHF}} as the NLHF dataset.
For Nash-MD-PG, we make use of the TRL library.
We observed an implementation mistake of \texttt{NashMDTrainer} in 0.13.0 version of TRL which results in wrong sampling policy when the policy is using PEFT \citep{peft}, so we addressed this issue while inheriting all other parts of the code in our local trainer.
For MPO, we use the official implementation kindly provided by the authors with slight modifications to support general preferences instead of BT models.
For all other algorithms, we implement our own trainers under the online IPO formulation, inheriting the \texttt{OnlineDPOTrainer} class in TRL library.

In NLHF training, the regularization coefficient $\beta$ is set to be $0.1$.
All models use LoRA \citep{hu2022lora}, with detailed settings listed in \Cref{tab:nlhf_hparam_shared,tab:nlhf_hparam}.

When running on $8\times$A6000 GPUs, one epoch takes Online IPO 1 (OMD) $1.51$ hrs, Online IPO 2 $0.98$ hrs, NashMD $3.48$ hrs, NashMD-PG $4.48$ hrs, and EGPO $1.56$ hrs (one effective epoch takes $2\times$ time).
Online IPO 2 is the most time-efficient, as it requires only $2$ (v.s. $3$) rollouts per prompt due to a variance reduction technique.
NashMD and NashMD-PG are less efficient because they require sampling from a geometric mixture of two policies.

When using the same micro batch size of $8$, EGPO consumes around $33$G of GPU memory per GPU, while all other algorithms consume around $28$G.
This difference occurs because EGPO backs-up gradients and optimizer states.

\begin{table*}[htbp]
    \centering
    
    \begin{tabular}{l | l}
        \hline
        \textbf{Shared hyperparameter} & \textbf{Shared value}  \\
        \hhline{=|=}
        LoRA & \\
        - $r$ & 256 \\
        - $\alpha$ & 512 \\
        - Dropout & $0.1$ \\
        \hline
        Number of epochs & $10$ \\
        \hline
        Train batch size & $64$ \\
        \hline
        Optimizer & AdamW \\
        - Gradient clipping norm & $1.0$ \\
        - $\beta_1, \beta_2$ & $0.9, 0.999$ \\
        - $\epsilon$ & $1 \times 10^{-6}$ \\
        - Weight decay & $0.01$ \\
        \hline
        Learning rate scheduler & WarmupDecayLR \\
        - Warmup max lr & $5 \times 10^{-7}$ \\
        - Warmup steps & $1000$ \\
        - Warmup type & Linear \\
        \hline
        Precision & bf$16$\\
        \hline
        Max new tokens & $64$ \\
        \hline
    \end{tabular}
    \caption{Shared hyperparameters of NLHF.}
    \label{tab:nlhf_hparam_shared}
\end{table*}

\begin{table*}[htbp]
    \centering
    
    \begin{tabular}{l | l | l | l | l | l | l}
        \hline
        \textbf{Hyperparameter} & \oipoone & \oipotwo & \nmd & \nmdpg & \mpo & \eg  \\
        \hhline{=|=|=|=|=|=|=}
        Sampling policy & & & & & & \\
        - Mixture coefficient $\gamma$ & $0.0$ & $1.0$ & $0.0$ & $1.0$ & \NA & $0.0$ \\
        - Temperature & $2.0$ & $1.0$ & $2.0$ & $1.0$ & $1.0$ & $2.0$ \\
        - Top\_k & $10$ & $0$ (all) & $10$ & $0$ (all) & $0$ (all) & $10$ \\
        - Top\_p & $1.0$ & $1.0$ & $1.0$ & $1.0$ & $0.9$ & $1.0$ \\
        \hline
        Alternate policy & & & & & & \\
        - Mixture coefficient & \NA & \NA & $0.125$ & $0.125$ & \NA & \NA \\
        \hline
    \end{tabular}
    \caption{Hyperparameters of NLHF.}
    \label{tab:nlhf_hparam}
\end{table*}



\para{Evaluation.}
We randomly sample $100$ prompts from the test split of PKU-SafeRLHF, and use each checkpoints to generate $10$ responses with \texttt{temperature} $=1$, \texttt{top\_k} $=100$, and \texttt{top\_p} $=0.95$.
For each pair of checkpoints under comparison, we calculate the average win-rates by querying the ground truth preference: $\wh{\cP} (\pi \succ \pi') = \frac{1}{1000} \sum_{i=1}^{100} \sum_{j=1}^{10} \cP (y_{i,j} \succ y_{i,j}' \mid x_i)$.

\subsubsection{Win-rates against the reference policy}

We train for $10$ epochs using each algorithm, so there are in total $60$ checkpoints.
Since pairwise win-rates are costly to compute for all the checkpoints, we first use the win-rates against the reference policy to select $2$ checkpoints from each algorithm, then do pairwise comparison on them.
\Cref{tab:res_safe_against_ref} records all the win-rates against the reference policy, namely $\cP (\pi_{\mathsf{ALG}}^{(k)} \succ \pi_\tref)$.

\begin{table}[htbp]
\begin{center}
\resizebox{\textwidth}{!}{
\begin{tabular}{cc|c||cc|c||cc|c||cc|c||cc|c||cc|c}
\toprule
ALG & Ep & $\pi_\tref$ & ALG & Ep & $\pi_\tref$ & ALG & Ep & $\pi_\tref$ & ALG & Ep & $\pi_\tref$ & ALG & Ep & $\pi_\tref$ & ALG & Ep & $\pi_\tref$ \\
\hhline{==================}
\multirow{10}{*}{\oipoone} & 1 & $55.5\%$ & \multirow{10}{*}{\oipotwo} & 1 & $54.7\%$ & \multirow{10}{*}{\nmd} & 1 & $56.7\%$ & \multirow{10}{*}{\nmdpg} & 1 & $53.3\%$ & \multirow{10}{*}{\mpo} & 1 & $57.7\%$ & \multirow{10}{*}{\eg} & 1 & $62.5\%$\\
 & 2 & $63.3\%$ &  & 2 & $60.6\%$ &  & 2 & $61.3\%$ &  & 2 & $52.0\%$ &  & 2 & $58.8\%$ &  & 2 & $70.3\%$\\
 & 3 & $66.7\%$ &  & 3 & $61.9\%$ &  & 3 & $65.3\%$ &  & 3 & $53.4\%$ &  & 3 & $57.2\%$ &  & 3 & $74.4\%$\\
 & 4 & $68.0\%$ &  & 4 & $65.4\%$ &  & 4 & $68.4\%$ &  & 4 & $\red{\boldsymbol{55.2\%}}$ &  & 4 & $58.9\%$ &  & 4 & $75.7\%$\\
 & 5 & $70.1\%$ &  & 5 & $64.8\%$ &  & 5 & $69.8\%$ &  & 5 & $54.3\%$ &  & 5 & $58.0\%$ &  & 5 & $\red{\boldsymbol{76.9\%}}$\\
 & 6 & $\red{\boldsymbol{72.8\%}}$ &  & 6 & $\red{\boldsymbol{66.8\%}}$ &  & 6 & $70.8\%$ &  & 6 & $55.0\%$ &  & 6 & $58.5\%$ &  & 6 & $76.4\%$\\
 & 7 & $70.2\%$ &  & 7 & $63.3\%$ &  & 7 & $72.1\%$ &  & 7 & $54.6\%$ &  & 7 & $\red{\boldsymbol{71.9\%}}$ &  & 7 & $75.7\%$\\
 & 8 & $\red{\boldsymbol{71.8\%}}$ &  & 8 & $66.2\%$ &  & 8 & $\red{\boldsymbol{72.8\%}}$ &  & 8 & $\red{\boldsymbol{55.1\%}}$ &  & 8 & $\red{\boldsymbol{70.2\%}}$ &  & 8 & $\red{\boldsymbol{77.4\%}}$\\
 & 9 & $71.4\%$ &  & 9 & $\red{\boldsymbol{66.3\%}}$ &  & 9 & $72.7\%$ &  & 9 & $53.2\%$ &  & 9 & $67.7\%$ &  & 9 & $74.9\%$\\
 & 10 & $70.7\%$ &  & 10 & $65.2\%$ &  & 10 & $\red{\boldsymbol{72.9\%}}$ &  & 10 & $53.1\%$ &  & 10 & $66.3\%$ &  & 10 & $75.2\%$\\
\bottomrule
\end{tabular}
}

\end{center}
\caption{Win-rates against the reference policy, $\pi_\tref$, evaluated by the ground truth preference on PKU-SafeRLHF. Each number is the win-rate of the \textbf{row} model against the \textbf{column} model. \textbf{Abbreviations:} ``Ep'' stands for the epoch number; ``\oipoone'' stands for ``Online IPO 1 (OMD)''; ``\oipotwo'' stands for ``Online IPO 2''; ``\nmd'' stands for ``Nash-MD''; ``\nmdpg'' stands for ``Nash-MD-PG''; ``\mpo'' stands for ``magnetic preference optimization''; ``\eg'' stands for ``Extragradient preference optimization''. Top $2$ highest win-rates of each algorithm's checkpoints are boldfaced red texts.}
\label{tab:res_safe_against_ref}
\end{table}

\subsubsection{Examples of generation results} \label{sec:gen_results}

We give some generation results of the models trained by different algorithms in \Cref{tab:gen_saferlhf_1,tab:gen_saferlhf_2}.
The prompts are from the test split of PKU-SafeRLHF.

\begin{table*}[htbp]
\footnotesize
    \centering
        \centering
        \small
        \scalebox{.8}{
\begin{tabular}{p{1.5cm}|p{15cm}}
\toprule[1.2pt]
\textbf{Prompt:} & \lstinputlisting[language=text]{prompt4}\\
\midrule
\textbf{Responses:} & \red{\textbf{Warning: Contents below may contain harmful language.}}\\
\multicolumn{2}{l}{\ }\\
$\pi_\tref$ & \lstinputlisting[language=text]{pitref4}\\
\oipoone & \lstinputlisting[language=text]{oipoone4}\\
\oipotwo & \lstinputlisting[language=text]{oipotwo4}\\
\nmd & \lstinputlisting[language=text]{nmd4}\\
\nmdpg & \lstinputlisting[language=text]{nmdpg4}\\
\mpo & \lstinputlisting[language=text]{mpo4}\\
\eg & \lstinputlisting[language=text]{eg4}\\
\bottomrule[1.2pt]
\end{tabular}
    }
\caption{Generation results: example 1.}
\label{tab:gen_saferlhf_1}
\end{table*}

\begin{table*}[htbp]
\footnotesize
    \centering
        \centering
        \small
        \scalebox{.8}{
\begin{tabular}{p{1.5cm}|p{15cm}}
\toprule[1.2pt]
\textbf{Prompt:} & \lstinputlisting[language=text]{prompt7}\\
\midrule
\textbf{Responses:} & \red{\textbf{Warning: Contents below may contain harmful language.}}\\
\multicolumn{2}{l}{\ }\\
$\pi_\tref$ & \lstinputlisting[language=text]{pitref7}\\
\oipoone & \lstinputlisting[language=text]{oipoone7}\\
\oipotwo & \lstinputlisting[language=text]{oipotwo7}\\
\nmd & \lstinputlisting[language=text]{nmd7}\\
\nmdpg & \lstinputlisting[language=text]{nmdpg7}\\
\mpo & \lstinputlisting[language=text]{mpo7}\\
\eg & \lstinputlisting[language=text]{eg7}\\
\bottomrule[1.2pt]
\end{tabular}
    }
\caption{Generation results: example 2.}
\label{tab:gen_saferlhf_2}
\end{table*}

\subsubsection{Additional experiment}

We conducted an additional experiment based on \texttt{Qwen/Qwen2.5-1.5B-Instruct}.
The reference policy is SFTed on the OpenHermes dataset\footnote{\url{https://huggingface.co/datasets/RLHFlow/SFT-OpenHermes-2.5-Standard}}.
The preference model and hyperparameters (batch size, PEFT config, etc.) are the same as in \Cref{sec:original_exp}.
Then we use a subset of an OpenRLHF dataset\footnote{\url{https://huggingface.co/datasets/OpenRLHF/prompt-collection-v0.1}} (the same dataset used in Appendix C.3 of \citet{wang2024magneticpreferenceoptimizationachieving}), and evaluated using a disjoint subset of the same dataset.

Since the code in \citet{wang2024magneticpreferenceoptimizationachieving} does not provided support for this dataset and their support for Safe-RLHF is hardcoded, we skipped MPO in this additional experiment.
The results are shown in \Cref{tab:res_open}.
It can be seen that only Online IPO 2, Nash-MD (not -PG), and \eg are able to get non-trivial win-rates.
\eg still outperforms all the baselines.

\begin{table}[t]
\begin{center}
\resizebox{\textwidth}{!}{
\begin{tabular}{cc||c|cc|cc|cc|cc|cc}
\toprule
ALG &  & \multirow{2}{*}{$\pi_\tref$} & \multicolumn{2}{c|}{\oipoone} & \multicolumn{2}{c|}{\oipotwo} & \multicolumn{2}{c|}{\nmd} & \multicolumn{2}{c|}{\nmdpg} & \multicolumn{2}{c}{\eg} \\
 & Ep & & 7 & 9 & 5 & 8 & 6 & 7 & 2 & 4 & 8 & 10 \\
\hhline{=============}
\multirow{2}{*}{\oipoone} & 7 & $\red{\boldsymbol{50.5\%}}$ & & & $49.8\%$ & $49.8\%$ & $49.8\%$ & $49.8\%$ & $\red{\boldsymbol{51.3\%}}$ & $\red{\boldsymbol{51.1\%}}$ & $45.5\%$ & $44.8\%$\\
 & 9 & $\red{\boldsymbol{50.1\%}}$ & & & $49.1\%$ & $49.8\%$ & $49.4\%$ & $49.4\%$ & $\red{\boldsymbol{52.4\%}}$ & $\red{\boldsymbol{52.2\%}}$ & $46.8\%$ & $47.0\%$\\
\hline
\multirow{2}{*}{\oipotwo} & 5 & $49.1\%$ & $\red{\boldsymbol{50.2\%}}$ & $\red{\boldsymbol{50.9\%}}$ & & & $49.5\%$ & $49.6\%$ & $\red{\boldsymbol{51.9\%}}$ & $\red{\boldsymbol{51.7\%}}$ & $46.0\%$ & $45.5\%$\\
 & 8 & $\red{\boldsymbol{50.3\%}}$ & $\red{\boldsymbol{50.2\%}}$ & $\red{\boldsymbol{50.2\%}}$ & & & $49.2\%$ & $49.5\%$ & $\red{\boldsymbol{52.5\%}}$ & $\red{\boldsymbol{51.1\%}}$ & $45.7\%$ & $45.0\%$\\
\hline
\multirow{2}{*}{\nmd} & 6 & $\red{\boldsymbol{50.7\%}}$ & $\red{\boldsymbol{50.2\%}}$ & $\red{\boldsymbol{50.6\%}}$ & $\red{\boldsymbol{50.5\%}}$ & $\red{\boldsymbol{50.8\%}}$ & & & $\red{\boldsymbol{52.5\%}}$ & $\red{\boldsymbol{51.1\%}}$ & $44.7\%$ & $46.2\%$\\
 & 7 & $\red{\boldsymbol{51.7\%}}$ & $\red{\boldsymbol{50.2\%}}$ & $\red{\boldsymbol{50.6\%}}$ & $\red{\boldsymbol{50.4\%}}$ & $\red{\boldsymbol{50.5\%}}$ & & & $\red{\boldsymbol{52.0\%}}$ & $\red{\boldsymbol{51.8\%}}$ & $47.0\%$ & $46.8\%$\\
\hline
\multirow{2}{*}{\nmdpg} & 2 & $49.1\%$ & $48.7\%$ & $47.6\%$ & $48.1\%$ & $47.5\%$ & $47.5\%$ & $48.0\%$ & & & $44.0\%$ & $43.1\%$\\
 & 4 & $49.5\%$ & $48.9\%$ & $47.8\%$ & $48.3\%$ & $48.9\%$ & $48.9\%$ & $48.2\%$ & & & $43.6\%$ & $44.0\%$\\
\hline
\multirow{2}{*}{\eg} & 8 & $\red{\boldsymbol{54.2\%}}$ & $\red{\boldsymbol{54.5\%}}$ & $\red{\boldsymbol{53.2\%}}$ & $\red{\boldsymbol{54.0\%}}$ & $\red{\boldsymbol{54.3\%}}$ & $\red{\boldsymbol{55.3\%}}$ & $\red{\boldsymbol{53.0\%}}$ & $\red{\boldsymbol{56.0\%}}$ & $\red{\boldsymbol{56.4\%}}$ & &\\
 & 10 & $\red{\boldsymbol{54.2\%}}$ & $\red{\boldsymbol{55.2\%}}$ & $\red{\boldsymbol{53.0\%}}$ & $\red{\boldsymbol{54.5\%}}$ & $\red{\boldsymbol{55.0\%}}$ & $\red{\boldsymbol{53.8\%}}$ & $\red{\boldsymbol{53.2\%}}$ & $\red{\boldsymbol{56.9\%}}$ & $\red{\boldsymbol{56.0\%}}$ & &\\
\bottomrule
\end{tabular}
}
\end{center}
\caption{Pairwise win-rates evaluated by the ground truth preference on the OpenRLHF dataset. Each number is the win-rate of the \textbf{row} model against the \textbf{column} model. \textbf{Abbreviations:} ``Ep'' stands for the epoch number; ``\oipoone'' stands for ``Online IPO 1 (OMD)''; ``\oipotwo'' stands for ``Online IPO 2''; ``\nmd'' stands for ``Nash-MD''; ``\nmdpg'' stands for ``Nash-MD-PG''; ``\eg'' stands for ``Extragradient preference optimization''. Win-rates larger than $50\%$ are boldfaced red texts.}
\label{tab:res_open}
\end{table}

\end{document}